%% file: acl_latex.tex
\pdfoutput=1
\documentclass[11pt]{article}

\usepackage[final]{acl}

\usepackage{times}
\usepackage{latexsym}

\usepackage{colortbl}
\usepackage{arydshln}
\usepackage{multirow}
\usepackage{booktabs}

\usepackage{amsmath}
\usepackage{amssymb}
\usepackage{mathtools}
\usepackage{amsthm}
\usepackage{thmtools}

\input{math_commands.tex}

\usepackage[T1]{fontenc}
\usepackage[utf8]{inputenc}

\usepackage{microtype}

\usepackage{inconsolata}

\usepackage{graphicx}

\title{Ultra-Low-Dimensional Prompt Tuning via Random Projection}

\author{\scalebox{0.97}{Zijun Wu, Yongchang Hao, Lili Mou}$^{*}$ \\
  \scalebox{0.97}{Dept. Computing Science \& Alberta Machine Intelligence Institute (Amii), University of Alberta}\\
  $^{*}$\scalebox{0.97}{Canada CIFAR AI Chair}\\
  \scalebox{0.97}{\texttt{\{zijun4,yongcha1\}@ualberta.ca}, \texttt{doublepower.mou@gmail.com}}}

\begin{document}
\maketitle
\begin{abstract}
Large language models achieve state-of-the-art performance but are increasingly costly to fine-tune. Prompt tuning is a parameter-efficient fine-tuning method that addresses parameter-efficiency by learning prompt embeddings, but these embeddings are typically tied to the model’s hidden dimensionality, limiting parameter saving. In this paper, we propose Ultra-Low-dimensional Prompt Tuning (ULPT), a simple yet effective method that optimizes prompts in a low-dimensional space (e.g., 2D) and uses a frozen random matrix for up-projection. ULPT can achieve 98\% reduction in the training parameters compared to vanilla prompt tuning while preserving performance. Our extensive experiments across over 20 NLP tasks demonstrate that ULPT consistently outperforms recent parameter-efficient tuning methods using significantly fewer parameters, making it well-suited as a storage-efficient framework for massive LLM customization.\footnote{Our code is available at \url{https://github.com/MANGA-UOFA/ULPT}}
\end{abstract}

\section{Introduction}
Fine-tuning large language models (LLMs) is essential for adapting them to specific tasks and controlling their outputs~\citep{raffel2020exploring, wei2022finetuned}. However, the enormous size of LLMs makes full fine-tuning prohibitively resource intensive, as it involves updating millions or even billions of parameters. To address this challenge, parameter-efficient fine-tuning methods have emerged as practical solutions, such as low-rank adaptation~(LoRA; \citeauthor{hu2022lora}, \citeyear{hu2022lora}) and prompt tuning~\citep{lester-etal-2021-power, li-liang-2021-prefix}. These methods drastically reduce the number of tunable parameters, offering an efficient alternative while achieving performance comparable to full fine-tuning.

Prompt tuning introduces learnable prompt embeddings exclusively in the input layer of the model~\citep{lester-etal-2021-power, LIU2024208}, automating prompt engineering by gradient descent to guide the frozen LLM in producing task-specific outputs~\citep{petrov2024when, petrov2024prompting}. By contrast, LoRA modifies the model by injecting low-rank weight matrices into its layers, causing the number of trainable parameters to scale with model's depth~\citep{hu2022lora}. Given that LLMs have encoded substantial knowledge during pretraining~\citep{NEURIPS2020_1457c0d6, NEURIPS2022_8bb0d291} and that both in-context learning and expertly crafted prompts can achieve remarkable results~\citep{NEURIPS2022_9d560961, dong-etal-2024-survey}, prompt tuning offers a more efficient and effective alternative to LoRA in many scenarios~\citep{shi2024dept}. 

Despite these advantages, a key limitation remains: prompt embeddings are typically constrained to match the model’s hidden dimensionality~\citep{lester-etal-2021-power, li-liang-2021-prefix, liu-etal-2022-p, choi-etal-2023-smop, razdaibiedina-etal-2023-residual}. As the size of the model increases, the dimensionality of the prompt embedding space also increases~\cite{raffel2020exploring, touvron2023llama}. This scaling leads to unnecessary complexity, as full dimensionality is often not required for task adaptation~\citep{aghajanyan-etal-2021-intrinsic, qin2022exploringuniversalintrinsictask}. Consequently, optimizing in this expanded space becomes inefficient in parameter's usage and may also increase the risk of overfitting, especially for less complex tasks or with limited training data.

In this paper, we propose \textbf{U}ltra-\textbf{L}ow-Dimensional \textbf{P}rompt \textbf{T}uning (ULPT), a novel approach that decouples the prompt and model dimensions, which enables to learn prompt embeddings in an ultra-low-dimensional space. A na\"ive attempt is to jointly optimize the ultra-low-dimensional embeddings with an up-projection matrix~\citep{xiao-etal-2023-decomposed, guo2024loptlowrankprompttuning}, but the learnable up-projection matrix may result in more trainable parameters than vanilla prompt tuning. 
We avoid this overhead by employing a \textit{random} but \textit{frozen} matrix for the up-projection, as shown in Figure~\ref{fig: ulpt diagram}a. We further introduce lightweight, learnable \textit{shift} and \textit{scale} embedding vectors to better align the randomly projected embeddings with the model’s prompt space~\citep{wu2024zeroshot}. 

We provide a convergence analysis for ULPT, and further show that a low-dimensional space with random projection can effectively approximate high-rank information that preserve the relational structure of embeddings, which is crucial for attention mechanisms in LLMs that depend on pairwise dot products between embeddings~\cite{NIPS2017_3f5ee243}.

\begin{figure}[!t]
\centering
\includegraphics[width=0.48\textwidth]{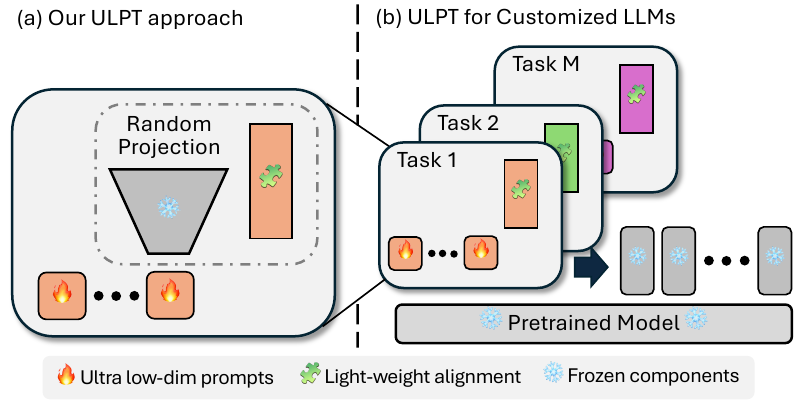} 
\caption{Overview of our approach. (a) ULPT up-projects ultra-low-dimensional embeddings with a random but fixed matrix. (b) ULPT can significantly reduce parameters storage for LLMs customization.}
\label{fig: ulpt diagram}
\end{figure}

In addition, the random projection into ultra-low dimensions introduces a controllable tradeoff between prompt dimension and length under a fixed parameter budget. We empirically demonstrate that allocating more tokens with lower-dimensional embeddings yields greater expressivity than using fewer high-dimensional tokens. This makes ULPT well-suited for massive LLM customization, such as per-user tuning while keeping storage footprints minimal, as shown in Figure~\ref{fig: ulpt diagram}b.

We evaluated ULPT across over 20 NLP tasks, including GLUE~\cite{wang-etal-2018-glue} and SuperGLUE~\cite{wang2019superglue} for language understanding, MRQA~\cite{fisch-etal-2019-mrqa} for question answering, GSM8K~\cite{cobbe2021gsm8k} and MBPP~\cite{austin2021program} for complex reasoning, as well as four additional tasks covering commonsense reasoning.
The results demonstrate that prompt tuning via ultra-low-dimensional optimization matches or surpasses the performance of fully parameterized prompt tuning while saving up to 98\% of trainable parameters. With an appropriate dimension, ULPT outperforms recent parameter-efficient fine-tuning methods, while requiring much fewer trainable parameters. 

In summary, our main contributions include:
\begin{itemize}
    \item We introduce ULPT, which optimizes prompts in a low-dimensional space with a random up-projection, drastically reducing trainable parameters while maintaining performance.
    \item Theoretically, we show that ULPT effectively approximates high-rank structures, which preserves embedding relational structures that are essential for attention mechanisms in LLMs.
    \item Empirically, we demonstrate that ULPT matches or surpasses vanilla prompt tuning across over 20 NLP tasks while saving trainable parameters by up to 98\%. Scaling to higher dimensions for optimization, it outperforms recent efficient tuning methods with much fewer trainable parameters.
\end{itemize}

\section{Related Work}
\paragraph{Parameter-efficient fine-tuning.}
With the rapid growth of pretrained neural networks, researchers have investigated parameter-efficient fine-tuning methods that update only a small set of parameters while maintaining high performance. One straightforward way is to tune specific components of the model. For example, BitFit updates only the bias terms~\cite{ben-zaken-etal-2022-bitfit}, and LayerNorm tuning only trains the layer-norm parameters~\cite{zhao2024tuning}.
Another line of work involves introducing and training small, task-specific non-linear modules, such as Adapters~\cite{pmlr-v97-houlsby19a} and AdapterDrop~\cite{ruckle-etal-2021-adapterdrop}. Other methods steer the activation representations either globally~\cite{wu2024reft, pan2024lisa} or locally~\cite{ yin2024lofit}.

Two prominent paradigms are low-rank adaptation (LoRA; \citeauthor{hu2022lora}, \citeyear{hu2022lora}) and prompt tuning methods~\cite{lester-etal-2021-power}, which are more related to our work. They will be further elaborated below.

\paragraph{Low-rank adaptation.} \citet{hu2022lora} assume that weight updates can be approximated by low-rank matrices and propose a low-rank adaptation (LoRA) method for fine-tuning a model.  Building upon this foundational work, many extensions have been developed to enhance LoRA’s performance. For example, ReLoRA~\cite{lialin2024relora} iteratively trains and merges low-rank adapters to achieve high-rank updates. 
\citet{hayou2024lora} propose learning low-rank matrices with different learning rates.
\citet{wu2024mixture} explore training a mixture of LoRA modules and leverage dynamic routing mechanisms for different task distributions or domains.

However, for large models, LoRA still requires a considerable number of trainable parameters, hindering efficient storage of task adaptations. To address this limitation, several studies have explored ways to further improve parameter efficiency. For example, VeRA~\cite{kopiczko2024vera} approximates LoRA weight update by using random matrices combined with two trainable scaling vectors. FourierFT~\cite{gao2024parameterefficient} learns a sparse set of spectral coefficients in the frequency domain and reconstructing weight updates via inverse Fourier transform. While these methods fine-tune in the weight space, our ULPT operates in the input embedding space and achieves better performance with lower training memory.

\paragraph{Prompt tuning.} 
\citet{shin-etal-2020-autoprompt} introduce the concept of learning prompt tokens to elicit knowledge from LLMs. Subsequently, \citet{lester-etal-2021-power} extend this idea to continuous prompt tuning, where prompt embeddings are optimized through gradient descent. 
Building on this, 
\citet{shi2024dept} observe that redistributing parameters to learn offsets for input token embeddings can enhance performance. \citet{lan-etal-2025-efficient} decompose prompt embeddings using SVD for more meaningful initialization. In parallel, multi-task prompt tuning has been explored, where the learned prompt parameters are reused across different tasks~\cite{wang2023multitask}. Closely related to our work, \citet{xiao-etal-2023-decomposed} decompose the prompt embedding matrix into two low-rank components: a low-dimensional prompt matrix and a learnable up-projection matrix. By contrast, our ULPT pushes parameter reduction further by using frozen and random up-projection matrix, which drastically lowers the number of trainable parameters while preserving performance. Our method is supported by random-projection theory~\cite{10.1145/502512.502546} and its recent applications in gradient compression~\cite{hao2024flora}.

\section{Methodology}
\subsection{Problem Formulation}
Prompt tuning introduces learnable token embeddings in the input layer of a language model~\cite{lester-etal-2021-power}. These embeddings are optimized via gradient descent based on the task-specific loss signals. During optimization, the model weights remain frozen, while the gradient is backpropagated to the input layer to update the learnable embeddings. Typically, learnable prompt embeddings  $\bm{e}_1, \cdots, \bm{e}_n \in \sR^d$ serve as a prefix~\cite{li-liang-2021-prefix}, followed by the text prompt, which is tokenized and represented by token embeddings $\bm{x}_1, \cdots, \bm{x}_l \in \sR^d$. Overall, the LLM has an input in the form of
\begin{align}
     (\bm{e}_1, \bm{e}_2, \cdots, \bm{e}_n, \bm{x}_1, \bm{x}_2, \cdots, \bm{x}_l)
\end{align}
where $n$ is a predefined prompt length and $l$ represents the length of the tokenized text.
The objective is to optimize the prompt embedding matrix $\mE \in \sR^{n \times d}$ over a given dataset \( \mathcal{D} \) based on the conditional log-likelihood:
\begin{align}
    \underset{\mE}{\arg\max} \sum_{(x, y) \in \mathcal{D}} \log P( y \mid \mE, x)
\end{align}
where \( (x, y) \in \mathcal{D} \) represents input--output pairs in a dataset.

\begin{figure}[!t]
\centering
\includegraphics[width=0.49\textwidth]{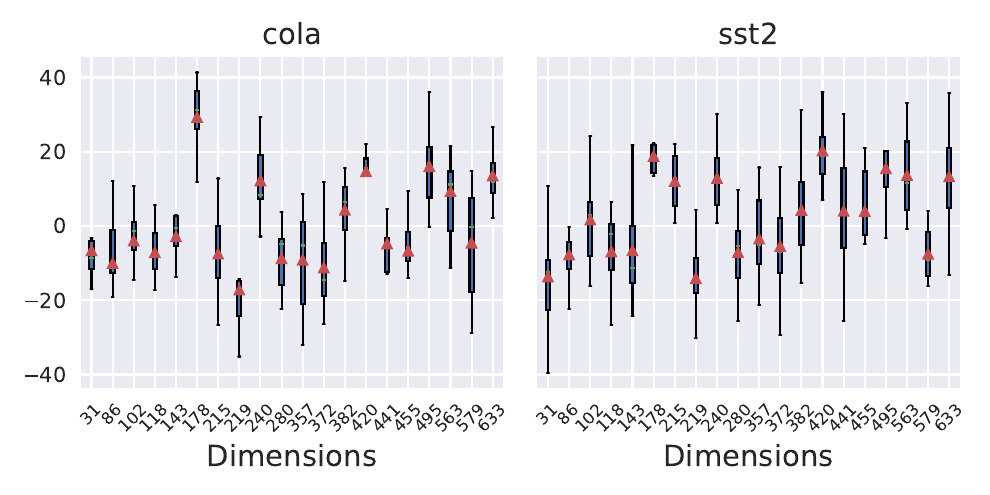} 
\caption{Distribution of prompt embedding values over 100 prompt tokens. We randomly selected $20$ dimensions from the original prompt embeddings, which have 768 dimensions as in the T5-base model.}
\label{fig: dims distribution}
\end{figure}

\subsection{Ultra-Low-Dimensional Prompt Tuning}\label{section: ulpt}
The learnable prompt embeddings do not inherently need to match the model dimension $\sR^d$ due to the low intrinsic dimensionality of downstream tasks~\cite{aghajanyan-etal-2021-intrinsic, qin2022exploringuniversalintrinsictask}. Inspired by low-rank adaptation~\cite{hu2022lora}, the prompt embedding matrix $\mE$ can be decomposed into the product of two matrices: $\mE = \mZ  \mP$, where $\mZ \in \sR^{n \times r}$ represents the prompt embeddings in an ultra-low $r$-dimensional space, and $\mP \in \sR^{r \times d}$ is a projection matrix that maps the low-dimensional embeddings back to the model's embedding space. 

A na\"ive implementation of this decomposition,  as in DPT~\cite{xiao-etal-2023-decomposed}, treats both $\mZ$ and $\mP$ as learnable parameters. This reduces the number of trainable parameters to $nr + rd$. However, the $rd$ term quickly becomes dominant as the model dimension $d$ grows or when separate $r \times d$ matrices must be maintained for multiple tasks, making DPT scale poorly in both parameter count and storage cost.

To address this limitation, we propose an ultra-low-dimensional prompt tuning (ULPT) method that only learns $r$-dimensional prompt embeddings $\bm Z$, while keeping the projection $\mP$ \textit{randomly initialized} and \textit{frozen} during training, denoted by $\tilde{\mP} \in \mathbb{R}^{r \times d}$. 
In implementation, we only need to store one single number---the random seed of a random number generator---to reconstruct $\tilde{\mP}$ when an LLM is loaded. 

In this way, we eliminate the need for storing the up-project matrix entirely, reducing the learnable parameters from $nr+rd$ to $nr$ (plus one extra random seed). Empirically, we find that freezing $\tilde{\mP}$ mitigates overfitting, particularly when fine-tuning on small datasets.

In our pilot study, we observe that typical prompt embeddings $\mE$, even without low-rank treatment, exhibit significant variation across different dimensions, as shown in Figure~\ref{fig: dims distribution}. For comparison, we report the distribution of pretrained embeddings in Appendix~\ref{appendix: Distribution of Pretrained Embeddings}, showing that they exhibit much smaller variation than the learned task-specific prompt embeddings. These variations may hinder effective training, therefore, we further introduce a learnable \textit{shift} embedding $\bm{b} \in \mathbb{R}^d$ and a learnable \textit{scale} embedding $\bm{s} \in \mathbb{R}^d$ to adjust the projected embeddings to ensure better alignment with the varying distributions across dimensions. Notice that the shift and scale embeddings are shared across different prompt token positions, but may vary for different tasks.

Specifically, an entry $\hat e_{ij}$ in the up-projected embedding matrix $\hat{\bm E}$ has the following form:
\begin{align}\label{eqn: reparameterization}
    \hat{e}_{ij} = \left( \sum_{k=1}^r z_{ik} \tilde {p}_{kj} \right) {s}_j + {b}_j,
\end{align}
where $z_{ik}$ and $\tilde p_{kj}$ are an entry in $\mZ$ and $\tilde\mP$ matrices, respectively; $s_j$ and $b_j$ are an entry in $\bm s$ and $\bm b$ vectors, respectively. 

Such a treatment introduces two $d$-dimensional vectors, resulting in the total number of trainable parameters being $nr + 2d$. This is significantly more parameter-efficient than full-dimension prompt tuning with $nd$-many parameters~\cite{lester-etal-2021-power} and vanilla low-rank prompt tuning with $(nr+rd)$-many parameters~\cite{xiao-etal-2023-decomposed}.

\subsection{Theoretical Analyses}
\input{theory.tex}

\begin{table*}[!t]
\centering
\resizebox{\textwidth}{!}{%
\begin{tabular}{lrccccccccccccccc}
\toprule
 & \multirow{2}{*}{\raisebox{-6mm}{\shortstack{\textbf{\#Param/} \\ \textbf{Task}}}} & \multicolumn{9}{c}{\textbf{GLUE}} & \multicolumn{6}{c}{\textbf{SuperGLUE}} \\
\cmidrule(l){3-11} \cmidrule(l){12-17}
\textbf{Method} & & \textbf{MNLI} & \textbf{QQP} & \textbf{QNLI} & \textbf{SST-2} & \textbf{STS-B} & \textbf{MRPC} & \textbf{RTE} & \textbf{CoLA} & \cellcolor{gray!20}\textbf{Avg.} & \textbf{MultiRC\!\!\!\!} & \textbf{Bool} & \textbf{WiC} & \textbf{WSC} & \textbf{CB} & \textbf{\cellcolor{gray!20}Avg.} \\
\midrule
\multicolumn{17}{c}{\textbf{Single-Task Learning}} \\
\midrule
Fine-tuning & 220M & 86.8 & 91.6 & 93.0 & 94.6 & 89.7 & 90.2 & 71.9 & 61.8 & \cellcolor{gray!20}84.9 & 72.8 & 81.1 & 70.2 & 59.6 & 85.7 & \cellcolor{gray!20}73.9 \\
Adapter & 1.9M & 86.5 & 90.2 & 93.2 & 93.8 & 90.7 & 85.3 & 71.9 & 64.0 & \cellcolor{gray!20}84.5 & 75.9 & 82.5 & 67.1 & 67.3 & 85.7 & \cellcolor{gray!20}75.7 \\
AdapterDrop & 1.1M & 86.3 & 90.2 & 93.2 & 93.6 & 91.4 & 86.3 & 71.2 & 62.7 & \cellcolor{gray!20}84.4 & 72.9 & 82.3 & 68.3 & 67.3 & 85.7 & \cellcolor{gray!20}75.3 \\
BitFit & 280K & 85.3 & 90.1 & 93.0 & 94.2 & 90.9 & 86.8 & 67.6 & 58.2 & \cellcolor{gray!20}83.3 & 74.5 & 79.6 & 70.0 & 59.6 & 78.6 & \cellcolor{gray!20}72.5 \\
LoRA & 3.8M & 86.3 & 89.0 & 93.2 & 94.3 & 90.0 & 90.1 & 75.5 & 63.3 & \cellcolor{gray!20}85.3 & 72.6 & 81.3 & 68.3 & 67.3 & 92.9 & \cellcolor{gray!20}76.5 \\
LST & 3.8M & 85.6 & 88.8 & 93.3 & 94.0 & 90.7 & 90.4 & 71.9 & 58.1 & \cellcolor{gray!20}84.1 & -- & -- & -- & -- & -- & \cellcolor{gray!20}-- \\
\hdashline
PT\textsuperscript{\dag} & 76.8K & 84.6 & 90.2 & 93.3 & 94.4 & 90.5 & 88.7 & 77.7 & 59.5 & \cellcolor{gray!20}84.9 & 72.3 & 80.4 & 67.7 & 67.3 & 78.6 & \cellcolor{gray!20}73.3 \\
DePT & 76.8K & 85.0 & 90.4 & 93.2 & 94.2 & 90.8 & 90.7 & 79.1 & 63.8 & \cellcolor{gray!20}85.9 & 74.3 & 79.3 & 68.7 & 67.3 & 92.9 & \cellcolor{gray!20}76.5 \\
DPT\textsuperscript{\dag}(r=10) & 9.0K & 84.4 & 90.2 & 93.3 & 94.6 & 91.2 & 87.7 & 77.7 & 57.8 & \cellcolor{gray!20}84.6 & 74.5 & 78.7 & 66.8  & 67.3 & 71.4 & \cellcolor{gray!20}71.7 \\
DPT\textsuperscript{\ddag}(r=64) & 55.6K & 85.2 & 90.3 & 92.9 & 93.6 & 90.4 & 88.2 & 79.1 & 63.5 & \cellcolor{gray!20}85.4 & 73.2 & 80.1 & 63.0  & 67.3 & 85.7  & \cellcolor{gray!20}73.9 \\
\hdashline
ULPT (r=2) & 1.7K & 81.9 & 90.3 & 92.3 & 92.9 & 89.8 & 89.2 & 76.3 & 59.5 & \cellcolor{gray!20}84.0 & 73.4 & 76.7 & 67.4 & 67.3 & 71.4 & \cellcolor{gray!20}71.2 \\
ULPT (r=16) & 3.1K & 82.9 & 90.0 & 93.1 & 93.8 & 90.5 & 89.2 & 80.6 & 54.3 & \cellcolor{gray!20}84.3 & 72.6 & 77.7 & 66.1 & 67.3 & 89.3 & \cellcolor{gray!20}74.6 \\
ULPT (r=64) & 7.9K & 84.9 & 90.3 & 93.1 & 93.5 & 90.7 & 90.2 & 81.3 & 63.7 & \cellcolor{gray!20}\textbf{86.0} & 73.1 & 78.2 & 69.0 & 67.3 & 96.4 & \cellcolor{gray!20}\textbf{76.8} \\
ULPT (r=256) & 27.1K & 85.5 & 90.3 & 92.8 & 94.3 & 90.6 & 90.7 & 76.3 & 63.7 & \cellcolor{gray!20}85.5 & 74.3 & 79.9 & 63.3 & 67.3 & 89.3 & \cellcolor{gray!20}74.8 \\
\midrule
\multicolumn{17}{c}{\textbf{Multi-Task Learning \& Transfer Learning}} \\
\midrule
Fine-tuning\textsuperscript{m} & 28M & 85.7 & 91.1 & 92.0 & 92.5 & 88.8 & 90.2 & 75.4 & 54.9 & \cellcolor{gray!20}83.8 & 74.4 & 81.1 & 70.0 & 71.2 & 85.7 & \cellcolor{gray!20}76.1 \\
Adapter\textsuperscript{m} & 1.8M & 86.3 & 90.5 & 93.2 & 93.0 & 89.9 & 90.2 & 70.3 & 61.5 & \cellcolor{gray!20}84.4 & 72.6 & 82.3 & 66.5 & 67.3 & 89.3 & \cellcolor{gray!20}75.6 \\
HyperFormer\textsuperscript{m} & 638K & 85.7 & 90.0 & 93.0 & 93.0 & 89.7 & 87.2 & 75.4 & 63.7 & \cellcolor{gray!20}84.8 & 72.9 & 82.5 & 69.0 & 67.3 & 85.7 & \cellcolor{gray!20}75.4 \\
HyperDecoder\textsuperscript{m} & 1.8M & 86.0 & 90.5 & 93.4 & 94.0 & 90.5 & 87.7 & 71.7 & 55.9 & \cellcolor{gray!20}83.7 & 70.4 & 78.8 & 67.1 & 61.5 & 82.1 & \cellcolor{gray!20}72.0 \\
SPoT\textsuperscript{t} & 76.8K & 85.4 & 90.1 & 93.0 & 93.4 & 90.0 & 79.7 & 69.8 & 57.1 & \cellcolor{gray!20}82.3 & 74.0 & 77.2 & 67.0 & 50.0 & 46.4 & \cellcolor{gray!20}62.9 \\
ATTEMPT\textsuperscript{t} & 232K & 84.3 & 90.3 & 93.0 & 93.0 & 89.7 & 85.7 & 74.3 & 57.4 & \cellcolor{gray!20}83.4 & 74.4 & 78.8 & 66.8 & 53.8 & 78.6 & \cellcolor{gray!20}70.5 \\
MPT\textsuperscript{t} & 77.6K & 85.9 & 90.3 & 93.1 & 93.8 & 90.4 & 89.1 & 79.4 & 62.4 & \cellcolor{gray!20}85.6 & 74.8 & 79.6 & 69.0 & 67.3 & 79.8 & \cellcolor{gray!20}74.1 \\
ATTEMPT\textsuperscript{t+m} & 96K & 83.8 & 90.0 & 93.1 & 93.7 & 90.8 & 86.1 & 79.9 & 64.3 & \cellcolor{gray!20}85.2 & 74.4 & 78.5 & 66.5 & 69.2 & 82.1 & \cellcolor{gray!20}74.1 \\
MPT\textsuperscript{t+m} & 10.5K & 84.3 & 90.0 & 93.0 & 93.0 & 90.4 & 89.2 & 82.7 & 63.5 & \cellcolor{gray!20}85.8 & 74.8 & 79.6 & 70.2 & 67.3 & 89.3 & \cellcolor{gray!20}76.1 \\
\bottomrule
\end{tabular}%
}
\caption{Performance on GLUE and SuperGLUE benchmarks based on the T5-base model. We report standard evaluation metrics, namely, Pearson correlation for STS-B, F1 for MultiRC, and accuracy for other classification tasks. 
\textsuperscript{\dag}We replicate prompt tuning (PT; \citeauthor{lester-etal-2021-power}, \citeyear{lester-etal-2021-power}) and DPT~\cite{xiao-etal-2023-decomposed} using their default configurations. Our replicated PT results slightly exceed those reported in previous studies. All other baseline results are directly sourced from \citet{shi2024dept}.  \textsuperscript{\ddag}The suggested rank for DPT is r=10 based on \citet{xiao-etal-2023-decomposed}; we replicate DPT with r=64 for controlled comparison with our ULPT.   \textsuperscript{t}Transfer learning methods.  \textsuperscript{m}Multi-task learning methods, whose ``\#param/task'' scores are calculated based on the GLUE benchmark.}
\label{tab:glue-superglue}
\end{table*}

\section{Experiments}

\subsection{Experiments on Language Understanding Tasks}\label{sec: Experiments on Language Understanding Tasks}

\paragraph{Datasets.} We evaluate ULPT across $21$ NLP tasks in our main experiment, following prior work~\cite{asai-etal-2022-attempt, wang2023multitask, shi2024dept}. Those tasks are grouped into 4 categories: (1) \textbf{GLUE} is a benchmark suite consisting of various language understanding tasks, such as MNLI~\cite{williams-etal-2018-broad}, QQP~\cite{wang-etal-2018-glue}, QNLI~\cite{demszky2018transforming}, SST-2~\cite{socher-etal-2013-recursive}, STS-B~\cite{cer-etal-2017-semeval}, MRPC~\cite{dolan-brockett-2005-automatically}, RTE~\cite{giampiccolo-etal-2007-third} and CoLA~\cite{warstadt2019neural_cola}. (2) \textbf{SuperGLUE} extends GLUE by considering more challenging tasks with limited training data, consisting of MultiRC~\cite{khashabi-etal-2018-looking}, BoolQ~\cite{clark-etal-2019-boolq}, WiC~\cite{pilehvar-camacho-collados-2019-wic}, WSC~\cite{levesque2012winograd_wnli}, and CB~\cite{de2019commitmentbank}. (3) The \textbf{MRQA} 2019 Shared Tasks are a set of QA tasks to test LLM generation capabilities, consisting of Natural Questions~\cite{kwiatkowski-etal-2019-natural}, HotpotQA~\cite{yang-etal-2018-hotpotqa}, SearchQA~\cite{dunn2017searchqa}, and NewsQA~\cite{trischler-etal-2017-newsqa}. (4) \textbf{Other classification tasks} beyond the above test suites are also considered, including WinoGrande~\cite{10.1145/3474381}, Yelp-2~\cite{10.5555/2969239.2969312}, SciTail~\cite{Khot_Sabharwal_Clark_2018}, and PAWS-Wiki~\cite{zhang-etal-2019-paws}. Further details on these datasets are provided in Table~\ref{tab: datasets details} in Appendix~\ref{appendix: details in datasets}. 

% We further evaluate ULPT on two reasoning tasks: GSM8K~\cite{cobbe2021gsm8k} and MBPP~\cite{austin2021program}. Details and results are provided in Section~\ref{sec: reasoning experiments}.

\paragraph{Baselines.}  
We evaluate ULPT against full-model fine-tuning, serving as a strong but parameter-intensive baseline. Second, we include state-of-the-art parameter-efficient methods such as Adapter~\cite{pmlr-v97-houlsby19a}, AdapterDrop~\cite{ruckle-etal-2021-adapterdrop}, BitFit~\cite{ben-zaken-etal-2022-bitfit}, HyperFormer~\cite{mahabadi2021parameter}, HyperDecoder~\cite{ivison-peters-2022-hyperdecoders}, LoRA~\cite{hu2022lora}, and Ladder Side-Tuning (LST;~\citeauthor{sunglst}, \citeyear{sunglst}). Third, we compare ULPT with vanilla prompt tuning (PT) and its variants: DePT learns offsets to the input token embeddings while using a separate learning rate for the prompt embeddings~\cite{shi2024dept}, and DPT is closely related to ULPT as it decomposes prompt embeddings into low-rank matrices~\cite{xiao-etal-2023-decomposed}, but it differs from ours by learning the up-projection. Finally, we compare ULPT with transfer or multi-task learning methods, including SPoT~\cite{vu-etal-2022-spot}, ATTEMPT~\cite{asai-etal-2022-attempt}, and MPT~\cite{wang2023multitask}. 

\paragraph{Implementation details.} 
We use the T5-base model with $d=768$. Consistent with prior work~\cite{shi2024dept, xiao-etal-2023-decomposed}, we set the number of prompt tokens $n =100$ for the prompt embeddings $\mZ \in \sR^{n \times r}$. For the rank $r$, we evaluate three ultra-low configurations $r = 2$, $16$, $64$ that achieve over 90\% dimensionality compression, and a more expressive configuration of $r=256$. We initialize the frozen up-projection matrix $\tilde{\mP}$ with a standard normal distribution and find that ULPT is robust to alternative random initializations such as uniform. In our analysis (Section~\ref{sec: analysis}), T5-small ($d=512$) and T5-large model ($d=1024$) are considered to evaluate the generality of ULPT across different model sizes and input space dimensions. Further training details are provided in Appendix~\ref{appendix: implementation for main experiments}.

\begin{table*}[!t]
  \centering
  \resizebox{\linewidth}{!}{%
    \begin{tabular}{l
                    rccccc
                    rccccc}
      \toprule
      & \multicolumn{6}{c}{\textbf{Llama 1B}} 
      & \multicolumn{6}{c}{\textbf{Llama 3B}} \\
      \cmidrule(lr){2-4} \cmidrule(lr){5-7} \cmidrule(lr){8-10} \cmidrule(lr){11-13}
      \textbf{Method}
      & \textbf{\#Param} & \textbf{VRAM} & \textbf{Runtime} & \cellcolor{gray!10}\textbf{GSM8K} & \cellcolor{gray!10}\textbf{MBPP} & \cellcolor{gray!20}\textbf{Avg.}
      & \textbf{\#Param} &  \textbf{VRAM} & \textbf{Runtime} & \cellcolor{gray!10}\textbf{GSM8K} & \cellcolor{gray!10}\textbf{MBPP} & \cellcolor{gray!20}\textbf{Avg.} \\
      \midrule
      ICL (4-shot)
      & -   & - & -  & \cellcolor{gray!10}34.3 & \cellcolor{gray!10}21.1 & \cellcolor{gray!20}27.7
      & -   & - & -  & \cellcolor{gray!10}62.5 & \cellcolor{gray!10}23.9 & \cellcolor{gray!20}43.2 \\
      \hdashline
      LoRA (r=1)
      & 106.5K & 10.22 & 215.1   & \cellcolor{gray!10}38.5 & \cellcolor{gray!10}26.7 &\cellcolor{gray!20}32.6
      & 286.7K & 18.44 & 615.8 & \cellcolor{gray!10}62.9 & \cellcolor{gray!10}32.1 &\cellcolor{gray!20}47.5 \\
      LoRA (r=4)
      & 426.0K & 10.22 & 215.7  & \cellcolor{gray!10}40.1 & \cellcolor{gray!10}27.2 &\cellcolor{gray!20}33.7
      & 1.15M  & 18.45 & 618.5 & \cellcolor{gray!10}63.4 & \cellcolor{gray!10}34.3 &\cellcolor{gray!20}48.9 \\
      LoRA (r=8)
      & 852.0K & 10.23  & 216.5  & \cellcolor{gray!10}40.2 & \cellcolor{gray!10}24.7 &\cellcolor{gray!20}32.5
      & 2.29M  & 18.46 & 621.1 & \cellcolor{gray!10}62.2 & \cellcolor{gray!10}37.8 &\cellcolor{gray!20}50.0 \\
      \hdashline
      VeRA (r=1)
      & 41.0K  & 10.02 & 216.0  & \cellcolor{gray!10}39.3 & \cellcolor{gray!10}24.4 &\cellcolor{gray!20}31.9
      & 114.7K & 17.97 & 620.9 & \cellcolor{gray!10}65.5 & \cellcolor{gray!10}35.5 &\cellcolor{gray!20}50.5 \\
      VeRA (r=4)
      & 41.1K  & 10.02 & 216.6 & \cellcolor{gray!10}39.6 & \cellcolor{gray!10}27.8 &\cellcolor{gray!20}33.7
      & 114.9K & 17.97 & 621.5 & \cellcolor{gray!10}65.0 & \cellcolor{gray!10}34.4 &\cellcolor{gray!20}49.7 \\
      VeRA (r=8)
      & 41.2K  & 10.02 & 217.0 & \cellcolor{gray!10}40.9 & \cellcolor{gray!10}29.5 &\cellcolor{gray!20}35.2
      & 115.1K & 17.97 & 621.6 & \cellcolor{gray!10}65.7 & \cellcolor{gray!10}33.9 &\cellcolor{gray!20}49.8 \\
      \hdashline
      FourierFT (n=128)
      & 4.1K   & 10.53 & 278.1  & \cellcolor{gray!10}35.8 & \cellcolor{gray!10}21.5 &\cellcolor{gray!20}28.7
      & 7.2K   & 19.80  & 880.9  & \cellcolor{gray!10}63.1 & \cellcolor{gray!10}21.9 &\cellcolor{gray!20}42.5 \\
      FourierFT (n=512)
      & 16.4K  & 10.53 & 278.5 & \cellcolor{gray!10}34.9 & \cellcolor{gray!10}27.3 &\cellcolor{gray!20}31.1
      & 28.7K  & 19.80 & 881.5 & \cellcolor{gray!10}66.6 & \cellcolor{gray!10}35.3 &\cellcolor{gray!20}\textbf{51.0} \\
      FourierFT (n=1024)
      & 32.8K  & 10.54 & 278.5 & \cellcolor{gray!10}36.6 & \cellcolor{gray!10}25.9 &\cellcolor{gray!20}31.3
      & 57.3K  & 19.80 & 882.6 & \cellcolor{gray!10}65.5 & \cellcolor{gray!10}35.4 &\cellcolor{gray!20}50.5 \\
      \hdashline
      IA3 & 147.5K & 10.83 & 229.2 & \cellcolor{gray!10}39.7	& \cellcolor{gray!10}26.4 & \cellcolor{gray!20}33.1 & 286.7K & 19.15 & 636.1 & \cellcolor{gray!10}63.7 & \cellcolor{gray!10}36.7 & \cellcolor{gray!20}50.2\\
      PT
      & 20.5K  & 9.79 & 203.6 & \cellcolor{gray!10}40.2 & \cellcolor{gray!10}24.7 &\cellcolor{gray!20}32.5
      & 30.7K  & 17.18 & 589.5 & \cellcolor{gray!10}65.3 & \cellcolor{gray!10}33.1 &\cellcolor{gray!20}49.2 \\
      \hdashline
      ULPT (r=2)
      & 4.1K   & 9.78 & 203.8 & \cellcolor{gray!10}39.7 & \cellcolor{gray!10}26.1 &\cellcolor{gray!20}32.9
      & 6.2K   & 17.17 & 589.2 & \cellcolor{gray!10}65.6 & \cellcolor{gray!10}35.9 &\cellcolor{gray!20}50.8 \\

      ULPT (r=64)
      & 4.7K   & 9.78 & 204.2  & \cellcolor{gray!10}42.4 & \cellcolor{gray!10}28.7 &\cellcolor{gray!20}\textbf{35.6}
      & 6.8K   & 17.17 & 589.6 &  \cellcolor{gray!10}65.4 & \cellcolor{gray!10}36.5 &\cellcolor{gray!20}\textbf{51.0} \\

      ULPT (r=256)
      & 6.7K   & 9.78 & 204.6 & \cellcolor{gray!10}41.4 & \cellcolor{gray!10}26.3 &\cellcolor{gray!20}33.9
      & 8.7K   & 17.17 & 589.6 & \cellcolor{gray!10}66.4 & \cellcolor{gray!10}34.6 &\cellcolor{gray!20}50.5 \\
      \bottomrule
    \end{tabular}%
  }
  \caption{Efficiency and performance of the Llama 3.2 models on GSM8K and MBPP. We report training parameters,  VRAM usages (GB), runtime (seconds/1K steps), GSM8K accuracy, and MBPP pass@1; In-context learning (ICL) serves as the lower-bound baseline.}
  \label{tab:ulpt-llama}
\end{table*}

% main results
\paragraph{Performance on GLUE and SuperGLUE.} 
As shown in Table~\ref{tab:glue-superglue}, ULPT achieves similar or higher performance on GLUE and SuperGLUE benchmark datasets compared with previous methods, while maintaining remarkable parameter efficiency. 

Profoundly, the extreme configuration of $r = 2$ retains at least $97\%$ performance of vanilla prompt tuning (PT) while saving 98\% of the parameters. This highlights the capability of ULPT and its advantage in large-scale LLM customization. To the best of our knowledge, we are the first to demonstrate effective prompt tuning in a two-dimensional parameter space.

When increasing the rank to $r=64$, ULPT outperforms that with $r=256$ and other state-of-the-art models. This suggests that ULPT effectively mitigates overfitting while remaining highly expressive. Specifically, the DPT model~\cite{xiao-etal-2023-decomposed} learns an up-projection matrix, underperforms ULPT despite using 7x more parameters at the same rank; even with the best setting $r=10$ suggested by the original paper~\cite{xiao-etal-2023-decomposed}, DPT is inferior to ULPT with $r = 64$ in both performance and efficiency.

ULPT also exhibits clear advantages in multi-task setups. A transfer learning method initializes a model by task mixtures and then adapts it to a specific task; therefore, it cannot save parameters. Previous studies of transfer learning include SPoT~\cite{vu-etal-2022-spot} and ATTEMPT~\cite{asai-etal-2022-attempt}. Our ULPT approach outperforms them in terms of accuracy and parameter efficiency, while offering a simpler training pipeline.
Multi-task learning, on the other hand, shares certain parameters across different tasks~\cite{mahabadi2021parameter,ivison-peters-2022-hyperdecoders,wang2023multitask}, and the parameter efficiency is measured on a per-task basis. Despite this, our ULPT still outperforms multi-task prompt tuning methods in both accuracy and per-task parameter efficiency.

\paragraph{Performance on MRQA and other classification tasks.} 
Following prior work on prompt tuning~\cite{lester-etal-2021-power,wang2023multitask,shi2024dept}, we evaluate ULPT on MRQA 2019 shared tasks as well as other language understanding tasks. Detailed results are provided in Appendix~\ref{appendix: mrqa results}. Consistent with our findings on GLUE and SuperGLUE, these results confirm that ULPT maintains competitive performance across diverse tasks while utilizing significantly fewer training parameters.

\subsection{Experiments on Reasoning Tasks}\label{sec: reasoning experiments}

\paragraph{Datasets and Baselines.}
Our previous experiments (Section~\ref{sec: Experiments on Language Understanding Tasks}) evaluate ULPT on four diverse task suites using the encoder–decoder T5 model, following prior prompt-tuning work~\cite{lester-etal-2021-power, wang2023multitask, shi2024dept}.

We further extend our experiments to two reasoning tasks, GSM8K~\cite{cobbe2021gsm8k} and MBPP~\cite{austin2021program}, using decoder-only Llama 3.2 models~\cite{grattafiori2024llama} at 1B and 3B scales, which have dimensionalities of $2048$, $3072$. 

We compare ULPT against in-context learning~\cite{NEURIPS2020_1457c0d6}, vanilla prompt tuning, IA3~\cite{liu2022fewshot}, LoRA and its recent ultra parameter-efficient variants: VeRA~\cite{kopiczko2024vera}, which approximates the LoRA update matrices via random projections, and FourierFT~\cite{gao2024parameterefficient}, which compresses weight updates by leveraging random spectral entries in the frequency domain. Implementation and training details are provided in Appendix~\ref{appendix: reasoning tasks}.

\paragraph{Performance on GSM8K and MBPP.}
As shown in Table~\ref{tab:ulpt-llama}, ULPT achieves the best trade-off between efficiency and performance. With only a few thousand trainable parameters, ULPT uses the least VRAM, achieves faster training runtime, and consistently delivers higher accuracy across both tasks and model scales. At $r = 64$, ULPT outperforms all baselines, including LoRA, VeRA, FourierFT, and vanilla prompt tuning. 

By contrast, IA3, LoRA and its variants require orders of magnitude more parameters to match or even underperform ULPT. Although FourierFT can match ULPT’s parameter scale at $n = 128$, it suffers from significantly lower accuracy, slower training speed, and higher memory usage. 

Finally, ULPT incurs no additional inference overhead. Since the learned prompt tokens are prepended once and cached during autoregressive generation, the decoding throughput remains virtually unchanged, detailed analysis can be found in Appendix~\ref{appendix: inference overhead}.

\subsection{In-Depth Analyses}\label{sec: analysis}

\begin{figure}[!t]
\centering
    \begin{minipage}[t]{0.23\textwidth}
        \centering
        \includegraphics[width=\textwidth]{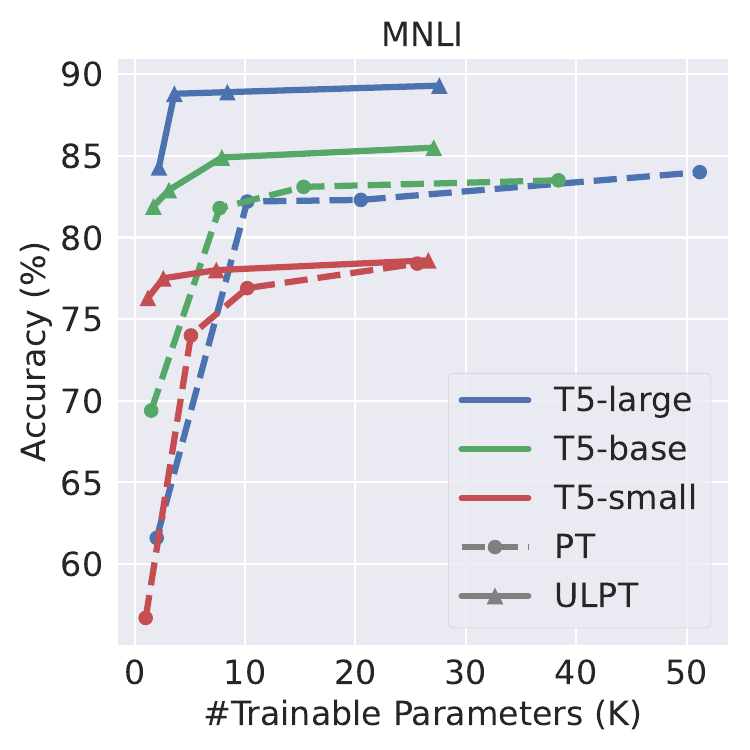}
    \end{minipage}
    \hfill
    \begin{minipage}[t]{0.23\textwidth}
        \centering
        \includegraphics[width=\textwidth]{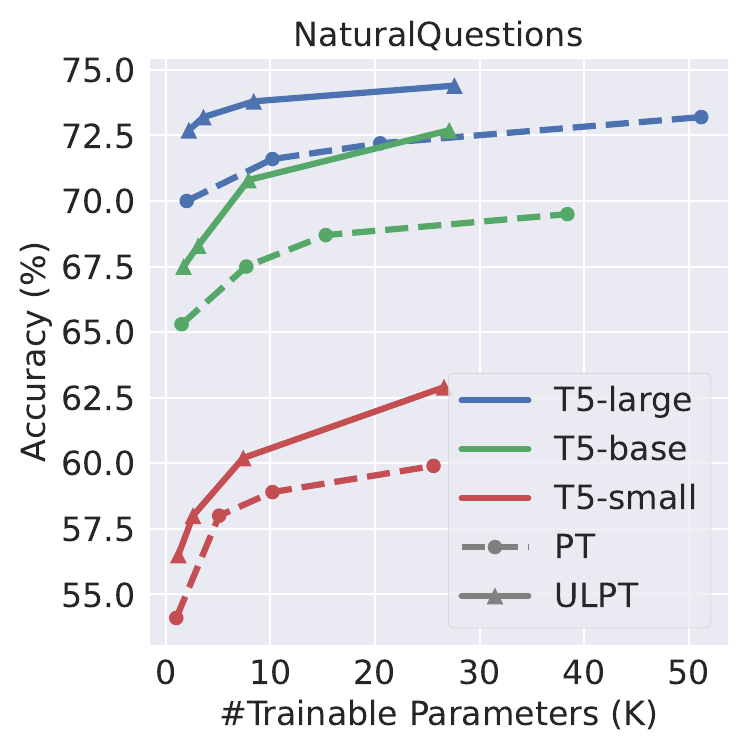}
    \end{minipage}
\caption{Results with controlled numbers of trainable parameters, suggesting that ULPT's longer prompt with lower dimensions offers more expressivity.}
\label{fig: control parameters}
\end{figure}

\begin{figure}[!t]
\centering
    \begin{minipage}[t]{0.23\textwidth}
        \centering
        \includegraphics[width=\textwidth]{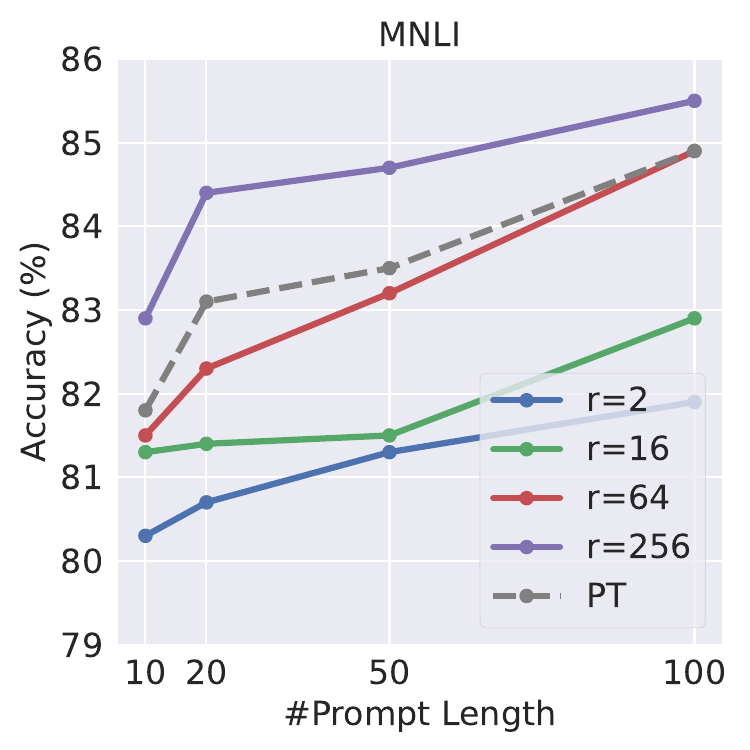}
    \end{minipage}
    \hfill
    \begin{minipage}[t]{0.23\textwidth}
        \centering
        \includegraphics[width=\textwidth]{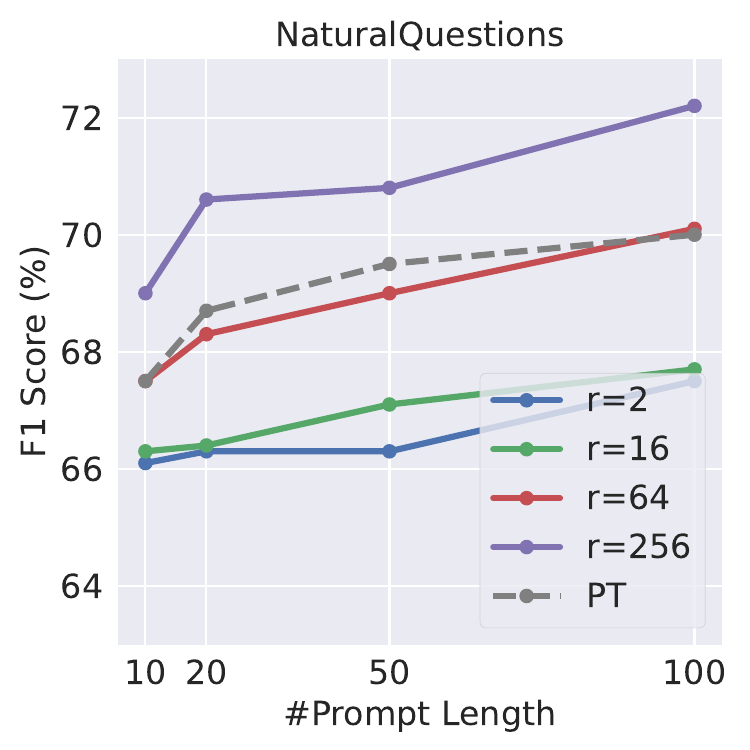}
    \end{minipage}
\caption{{Performance when the number of prompt tokens for both ULPT and na\"ive PT varies from $10$ to $100$.}}
\label{fig: control ULPT length}
\end{figure}

\paragraph{Dimension–length trade-off drives expressivity.} ULPT enables a trade-off between prompt length and dimension under a fixed parameter budget. To investigate this, we compare ULPT with vanilla prompt tuning when the learnable parameters are controlled. For ULPT, we fix the prompt token number at $100$ and vary the rank from $2$ to $256$; for vanilla full-dimensional prompt tuning, we vary the token number from $2$ to $50$. This analysis is conducted with two large datasets MNLI and Natural Questions across three model sizes: T5-small, T5-base, and T5-large.

Figure~\ref{fig: control parameters} illustrates the results, showing that our low-dimensional ULPT with more tokens (solid lines) always outperforms vanilla full-dimensional prompt tuning with fewer tokens (dashed lines). The analysis suggests that, when the number of learnable parameters is controlled, a longer prompt with a lower dimension offers more expressivity due to the additional Transformer steps. We provide additional results in Appendix~\ref{appendix: bloomz}, which confirms that this also holds for decoder-only models.

\paragraph{Longer prompts improve performance.} Recall that Table~\ref{tab:glue-superglue} has analyzed our ULPT performance with different ranks. We now vary the number of prompt tokens from $10$ to $100$ and plot the trend in Figure~\ref{fig: control ULPT length}. We see that our ULPT exhibits a similar trend as vanilla prompting: performance increases with a longer prompt. With an appropriate rank configuration, our ULPT consistently outperforms vanilla prompt tuning under different lengths.

\begin{figure}[!t]
\centering
    \begin{minipage}[t]{0.23\textwidth}
        \centering
        \includegraphics[width=\textwidth]{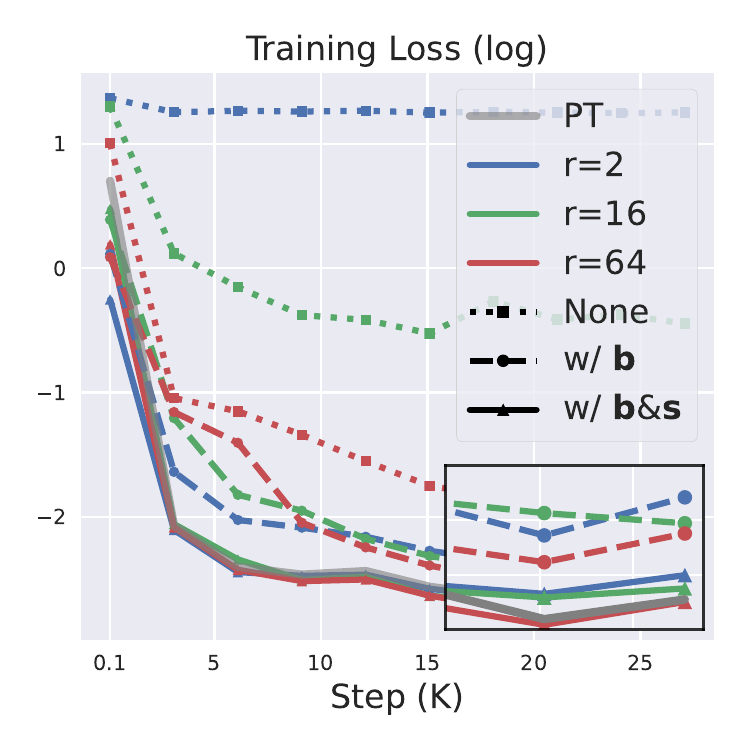}
    \end{minipage}
    \hfill
    \begin{minipage}[t]{0.23\textwidth}
        \centering
        \includegraphics[width=\textwidth]{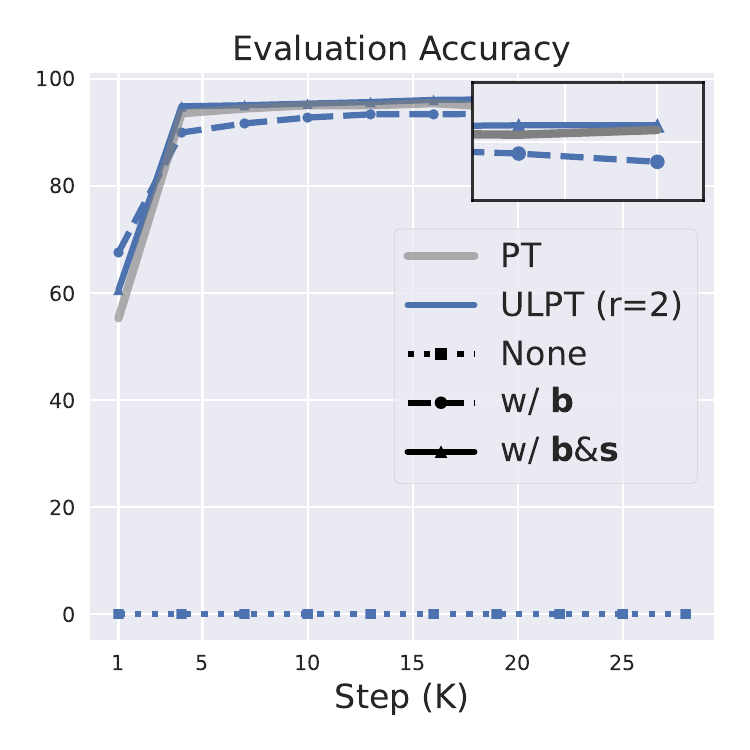}
    \end{minipage}
\caption{\textbf{Left:} Training loss curves comparing ULPT with no alignment (dotted), with learnable shift only (dashed), and with both shift and scale (solid). \textbf{Right:} Evaluation accuracy curves for ULPT at $r=2$. Adding shift significantly improves optimization and accuracy, while adding scale yields further gains. Trends are consistent across ranks.}
\label{fig: pilot sst2}
\end{figure}

\begin{figure}[!t]
\centering
\includegraphics[width=0.48\textwidth]{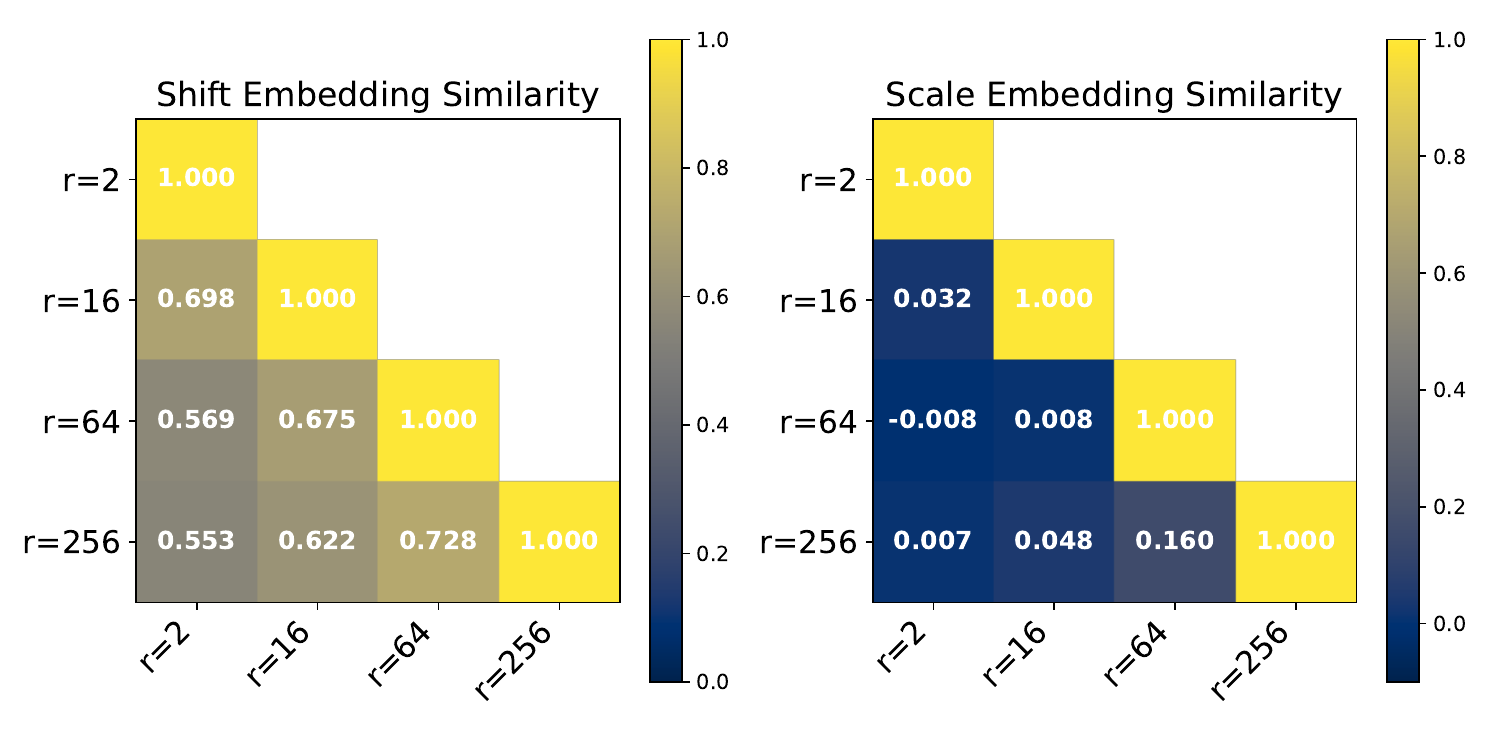}
\caption{\textbf{Left:} Shift embeddings learned with different ranks are highly similar, suggesting a general alignment role.
\textbf{Right:} Scale embeddings vary significantly, indicating their dependence on frozen random projections.}
\label{fig: similarity matrix}
\end{figure}

\paragraph{Ablation on shift and scale embeddings.}
We conduct an ablation study on the learnable \textit{shift} embedding $\bm{b} \in \mathbb{R}^d$ and \textit{scale} embedding $\bm{s} \in \mathbb{R}^d$, using the SST-2 dataset with the T5-base model as the testbed, where we set the token number to be $n=100$. The results are shown in Figure~\ref{fig: pilot sst2}. As seen, the dotted lines correspond to removing both shift and scale embeddings; their training loss remains high, suggesting that na\"ively freezing the projection matrix $\Tilde{\mP}$ hinders the optimization process and consequently lowers the model performance. Introducing a learnable shift embedding $\bm{b}$  provides a substantial improvement (dashed lines), particularly in the low-dimensional configuration of $r=2$. A learnable scale embedding \textit{scale} $\bm{s}$ further improves the training process and performance (solid lines). The ablation study shows that, although shift and scale embeddings are additional $2d$-many parameters, they play an important role in ultra-low-dimensional prompt tuning.

To further investigate the behavior of these embeddings, we analyze the pairwise cosine similarities of the shift $\bm{b}$ and scale $\bm{s}$ vectors under different rank configurations, visualized in Figure~\ref{fig: similarity matrix}. Interestingly, the learned shift embeddings show consistently high similarity scores with different rank configurations, indicating their primary role as a stable alignment mechanism after up-projection. By contrast, the scale embeddings show near-zero similarity, as they depend on the sampled (and frozen) random projection matrix $\tilde{\bm P}$.

\paragraph{Comparison with an alternative method of tuning projection matrix.}

\begin{table}[!t]
\centering
\resizebox{0.48\textwidth}{!}{%
\begin{tabular}{lccllll}
\toprule
& \multicolumn{2}{c}{\textbf{Train?}} & \multicolumn{4}{c}{\textbf{\#Parameters}} \\
\cmidrule(l){2-3} \cmidrule(l){4-7}
\textbf{Dataset} & $\mZ$ & $\mP$ & 1.7K & 3.1K & 7.9K & 27.1K \\
\midrule
\multirow{2}{*}{MNLI}&  & \checkmark & NF & 82.9 (2) & 84.5 (8) & 85.3 (33) \\
& \cellcolor{gray!10}\checkmark & \cellcolor{gray!10}& \cellcolor{gray!10}$\textbf{81.9}_{0.1}$ (2) & \cellcolor{gray!10}$\textbf{83.0}_{0.2}$ (16) & \cellcolor{gray!10}$\textbf{84.9}_{0.1}$ (64) & \cellcolor{gray!10}$\textbf{85.4}_{0.1}$ (256) \\
\midrule
\multirow{2}{*}{NQ} &  & \checkmark & NF & 66.9 (2) & 70.0 (8) & 72.0 (33) \\
& \cellcolor{gray!10}\checkmark & \cellcolor{gray!10} & \cellcolor{gray!10}$\textbf{67.2}_{0.2}$ (2) & \cellcolor{gray!10}$\textbf{68.0}_{0.4}$ (16) & \cellcolor{gray!10}$\textbf{70.7}_{0.3}$ (64) & \cellcolor{gray!10}$\textbf{72.6}_{0.2}$ (256) \\
\bottomrule
\end{tabular}}
\caption{Training either low-dimensional embeddings $\mZ$ (ours) or the up-projection $\mP$ with T5-base. Numbers in the brackets refer to the rank $r$ given the controlled number of parameters. ``NF'' refers to non-feasible.}
\label{tab: tune P or Qresults}
\end{table}

We follow the setup in Section~\ref{sec: analysis} to ablate which component to tune in the low-rank decomposition. The low-rank decomposition $\mE=\mZ \mP$ allows an alternative approach that freezes $\mZ$ and tunes $\mP$, which contrasts with our approach that freezes $\mP$ and tunes $\mZ$. The comparison is shown in Table~\ref{tab: tune P or Qresults}.
The alternative setup (tuning $\mP$) can be viewed as learning an up-projection from a set of random but frozen low-dimensional vectors. However, a key drawback of making $\mP$ trainable is the rapid growth in the number of parameters when the rank $r$ increases, since $d \gg n$ in most practical scenarios. To ensure a fair comparison, we control the number of parameters by varying the rank $r$ for both methods.

As seen, tuning $\mP$ fails to be feasible in the 1.7K-parameter setup. Even if we set $r=2$, tuning $\mP$ results in 3.1K parameters, equivalent to our $r=16$ setup.
With a larger budget, tuning $\mP$ achieves slightly worse performance than our ULPT which tunes $\mZ$. This analysis verifies the expressiveness of random projections; it also shows that our ULPT is superior to the alternative approach.

\section{Conclusion}
We introduce ULPT, a novel parameter-efficient method that learns task-specific prompts in a low-dimensional space and projects them into the model space using a frozen random matrix with learned shift and scale vectors. ULPT achieves competitive or superior performance compared with state-of-the-art parameter-efficient fine-tuning methods while using dramatically fewer trainable parameters.

\section{Limitations}
We have comprehensively evaluated ULPT on both encoder–decoder T5 models and decoder-only Bloomz and Llama models on various NLP and reasoning datasets, demonstrating its parameter efficiency and competitive performance. However, due to the limitation in computational resource, we have not explored its effectiveness on very large language models (ranging from tens to hundreds of billions of parameters). We anticipate that, for such large LLMs, parameter-efficient methods like ULPT are particularly suitable for lightweight customization, such as adapting generation style or output formatting~\cite{liu-etal-2024-customizing-large}, rather than unlocking new capabilities which are better addressed in pre-training. This highlights the advantage of ULPT in scenarios that require ultra-low parameter usage without compromising effectiveness.

\section*{Acknowledgments}
We thank the reviewers and area chairs for their efforts. The research is supported in part by the Natural Sciences and Engineering Research Council of Canada (NSERC), the Amii Fellow Program, the Canada CIFAR AI Chair Program, a donation from DeepMind, and the Digital Research Alliance of Canada (alliancecan.ca).

\bibliography{custom}
\clearpage
\appendix

\section{Experimental Details}
\subsection{Datasets Statistics for Language Understanding Tasks}\label{appendix: details in datasets}

\begin{table*}[!t]
\centering
\resizebox{\textwidth}{!}{%
\begin{tabular}{lrrrrrll}
\toprule
\textbf{Dataset} & \textbf{Source Length} & \textbf{Target Length} & \textbf{\#Train} & \textbf{\#Valid} & \textbf{\#Test} & \textbf{Type} & \textbf{Size}\\ 
\midrule
\multicolumn{7}{c}{\textbf{GLUE Benchmark}} \\
\midrule
MNLI   & 31.8  & 1.0 & 392,702 & 9,832  & 9,815  & Natural language inference     & Large          \\
QQP    & 24.1  & 1.0 & 362,846 & 1,000  & 40,431 & Paraphrasing   & Large     \\
QNLI   & 38.4  & 1.0 & 103,743 & 1,000  & 5,463  & Natural language inference     & Large        \\
SST-2  & 10.4  & 1.0 & 66,349  & 1,000  & 872    & Sentiment analysis   & Medium      \\
STS-B  & 21.9  & 1.0 & 5,749   & 750    & 750    & Sentence similarity & Small \\
MRPC   & 45.9  & 1.0 & 3,668   & 204    & 204    & Paraphrasing   &Small     \\
RTE    & 54.4  & 1.0 & 2,490   & 138    & 139    & Natural language inference       &Small        \\
CoLA   & 8.7   & 1.0 & 8,551   & 521    & 522    & Acceptability  &Small    \\
\midrule
\multicolumn{7}{c}{\textbf{SuperGLUE Benchmark}} \\
\midrule
MultiRC & 286.1 & 1.0 & 27,243  & 2,424  & 2,424  & Question answering & Medium   \\
BoolQ   & 108.3 & 1.0 & 9,427   & 1,635  & 1,635  & Question answering  & Small  \\
WiC     & 18.4  & 1.0 & 5,428   & 319    & 319    & Word sense disambiguation & Small \\
WSC     & 28.1  & 1.0 & 554     & 52     & 52     & Commonsense reasoning  & Small\\
CB      & 64.6  & 1.0 & 250     & 28     & 28     & Natural language inference  & Small             \\
% ReCoRD  & 210.7 & 1.5 & 137,484 & 1,370  & 15,176 & Common Sense Reason.  \\
\midrule
\multicolumn{7}{c}{\textbf{MRQA 2019 Shared Task}} \\
\midrule
NaturalQuestions & 242.7 & 4.5 & 103,071 & 1,000  & 12,836 & Question answering & Large\\
HotpotQA         & 225.7 & 2.6 & 71,928  & 1,000  & 5,901  & Question answering & Medium\\
SearchQA         & 942.8 & 2.0 & 116,384 & 1,000  & 16,980 & Question answering & Large\\
NewsQA           & 615.5 & 5.1 & 73,160  & 1,000  & 4,212  & Question answering & Medium\\
\midrule
\multicolumn{7}{c}{\textbf{Other Datasets}} \\
\midrule
WinoGrande      & 23.8  & 1.0 & 39,398  & 1,000  & 1,267  & Commonsense reasoning & Medium\\
YelpPolarity    & 134.0 & 1.0 & 100,000 & 1,000  & 38,000 & Sentiment analysis  & Large       \\
SciTail         & 30.8  & 1.0 & 23,596  & 652    & 652    & Natural language inference & Medium \\
PAWS            & 44.7  & 1.0 & 49,401   & 8,000  & 8,000  & Sentence Similarity  & Medium  \\
\bottomrule
\end{tabular}}
\caption{Dataset information and statistics for the main experiments in Section~\ref{sec: Experiments on Language Understanding Tasks}.}
\label{tab: datasets details}
\end{table*}

We present detailed information for the 21 natural language understanding tasks in Table~\ref{tab: datasets details}. Following previous work~\citep{wang2023multitask, shi2024dept}, we preprocess the labels for classification and multiple-choice tasks into a single-token label (e.g., 0, 1, 2, …) to simplify evaluation. For MRQA, the model generates an answer containing a sequence of tokens.

Based on the training set size, the tasks can be roughly categorized into three scales: small ($<$10K samples), medium (10--100K samples), and large ($>$100K samples). Notably, SuperGLUE contains small training sets, and is generally considered more challenging than GLUE, making it more susceptible to overfitting due to its limited samples. By contrast, MRQA and the tasks in the ``Others'' category consist of more complex tasks, likely requiring more parameters to capture their difficulty.

\subsection{Training Details for Language Understanding Tasks}\label{appendix: implementation for main experiments}
In our experiment for $21$ language understanding tasks (Section~\ref{sec: Experiments on Language Understanding Tasks}), we use a batch size of $16$ and a default learning rate of $6e{-1}$ with AdamW. The learning rate follows a linear schedule, warming up for $500$ steps and then decaying linearly to $0$. We set a maximum sequence length of $256$ for most tasks, except for SuperGLUE-MultiRC being $348$ and MRQA being $512$. ULPT is trained on all tasks for up to $100,000$ steps. Performance is evaluated every $1,000$ steps, with the best checkpoint selected based on the validation set.

\subsection{Implementation and Training Details for Reasoning Tasks}\label{appendix: reasoning tasks}
In Section~\ref{sec: reasoning experiments}, we evaluate ULPT on Llama 3.2~\cite{grattafiori2024llama}, a decoder-only model in 1B ($d=2048$) and 3B ($d=3072$) scales. We consider two reasoning tasks, GSM8K~\cite{cobbe2021gsm8k} for math reasoning and MBPP~\cite{austin2021program} for code generation. 

As baselines, we include in-context learning (4-shot demostrations sampled from each training set), LoRA and VeRA adapters with ranks $r\in\{1,4,8\}$, FourierFT with three frequency configurations $n\in\{128,512,1024\}$. Vanilla prompt tuning and ULPT both run with $10$ learnable tokens, since extending the length yields no additional benefit on Llama. We evaluate ULPT on three rank configurations of 2, 64 and 256. All methods share a batch size of 4, with training for 3 epochs on GSM8K and 10 epochs on MBPP. We set learning rates to 1e-3 for LoRA and VeRA, and 1e-2 for ULPT, prompt tuning and FourierFT.
At inference, we apply greedy decoding with the maximum generation length of $1024$ tokens.

\begin{table*}[!t]
\centering
\resizebox{0.88\textwidth}{!}{%
\begin{tabular}{lccccccccccc}
\toprule
 &  & \multicolumn{5}{c}{\textbf{MRQA}} & \multicolumn{5}{c}{\textbf{Others}} \\
\cmidrule(l){3-7} \cmidrule(l){8-12}
\textbf{Method} & \textbf{\#Param} & \textbf{NQ} & \textbf{HQA} & \textbf{SQA} & \textbf{NewsQA} & \cellcolor{gray!20}\textbf{Avg.} & \textbf{WG} & \textbf{Yelp} & \textbf{SciTail} & \textbf{PAWS} & \cellcolor{gray!20}\textbf{Avg.} \\
\midrule
Fine-tuning & 220M & 75.1 & 77.5 & 81.1 & 65.2 & \cellcolor{gray!20}\textbf{74.7} & 61.9 & 96.7 & 95.8 & 94.1 & \cellcolor{gray!20}\textbf{87.1} \\
Adapter & 1.9M & 74.2 & 77.6 & 81.4 & 65.6 & \cellcolor{gray!20}74.7 & 59.2 & 96.9 & 94.5 & 94.3 & \cellcolor{gray!20}86.2 \\
BitFit & 280K & 70.7 & 75.5 & 77.7 & 64.1 & \cellcolor{gray!20}72.0 & 57.2 & 94.7 & 94.7 & 92.0 & \cellcolor{gray!20}84.7 \\
LoRA & 3.8M & 72.4 & 62.3 & 72.5 & 56.9 & \cellcolor{gray!20}66.0 & 58.2 & 97.1 & 94.7 & 94.0 & \cellcolor{gray!20}86.0 \\
\hdashline
SPoT & 76.8K & 68.2 & 74.8 & 75.3 & 58.2 & \cellcolor{gray!20}69.1 & 50.4 & 95.4 & 91.2 & 91.1 & \cellcolor{gray!20}82.0 \\
ATTEMPT & 232K & 70.4 & 75.2 & 78.5 & 62.8 & \cellcolor{gray!20}71.4 & 57.6 & 96.7 & 93.1 & 92.1 & \cellcolor{gray!20}84.9 \\
PT\textsuperscript{\dag} & 76.8K & 70.0 & 74.7 & 75.3 & 63.0 & \cellcolor{gray!20}70.8 & 49.6 & 95.6 & 92.0  & 57.9 & \cellcolor{gray!20}73.8 \\
DPT\textsuperscript{\dag}(r=10) & 9.0K & 71.3 & 75.5 & 76.3 & 63.5 & \cellcolor{gray!20}71.7 & 49.6 & 96.1 & 95.6 & 92.2 & \cellcolor{gray!20}83.4 \\
DPT (r=256) & 222K & 71.4 & 76.0 & 77.6 & 64.2 & \cellcolor{gray!20}72.3 & 49.6 & 96.3 & 95.2 & 55.8 & \cellcolor{gray!20}74.2 \\
DePT & 76.8K & 73.2\textsubscript{0.3} & 76.0\textsubscript{0.2} & 77.6\textsubscript{0.2} & 64.4\textsubscript{0.1} & \cellcolor{gray!20}73.0 & 59.0\textsubscript{0.2} & 96.8\textsubscript{0.1} & 95.6\textsubscript{0.2} & 93.7\textsubscript{0.1} & \cellcolor{gray!20}86.3 \\
MPT & 77.6K & 72.0\textsubscript{0.1} & 75.8\textsubscript{0.1} & 77.2\textsubscript{0.1} & 63.7\textsubscript{0.1} & \cellcolor{gray!20}72.2 & 56.5\textsubscript{0.9} & 96.4\textsubscript{0.0} & 95.5\textsubscript{0.3} & 93.5\textsubscript{0.1} & \cellcolor{gray!20}85.5 \\
\hdashline
ULPT (r=2) & 1.7K & 67.2\textsubscript{0.2} & 74.0\textsubscript{0.1} & 71.7\textsubscript{0.2} & 61.4\textsubscript{0.1} & \cellcolor{gray!20}68.6 & 49.5\textsubscript{0.2} & 95.6\textsubscript{0.1} & 93.0\textsubscript{0.9} & 90.4\textsubscript{0.2} & \cellcolor{gray!20}82.1\\
ULPT (r=16) & 3.1K & 68.0\textsubscript{0.3} & 74.3\textsubscript{0.0} & 72.9\textsubscript{0.1} & 61.3\textsubscript{0.5} & \cellcolor{gray!20}69.1 & 52.3\textsubscript{0.9} & 95.6\textsubscript{0.2} & 93.1\textsubscript{0.7} & 90.5\textsubscript{0.3} & \cellcolor{gray!20}82.9 \\
ULPT (r=64) & 7.9K & 70.7\textsubscript{0.3} & 75.3\textsubscript{0.1} & 75.3\textsubscript{0.1} & 62.9\textsubscript{0.5} & \cellcolor{gray!20}71.1 & 56.6\textsubscript{0.9} & 96.2\textsubscript{0.1} & 94.4\textsubscript{0.9} & 91.7\textsubscript{0.4} & \cellcolor{gray!20}84.7 \\
ULPT (r=256) & 27.1K & 72.6\textsubscript{0.2} & 76.5\textsubscript{0.1} & 77.9\textsubscript{0.1} & 64.2\textsubscript{0.2} & \cellcolor{gray!20}72.8 & 57.6\textsubscript{0.8} & 96.6\textsubscript{0.2} & 96.2\textsubscript{0.1} & 93.0\textsubscript{0.1} & \cellcolor{gray!20}85.9 \\
\bottomrule
\end{tabular}%
}
\caption{Performance on MRQA and other benchmarks using the T5-base model. The standard metrics reported are the F1 score for MRQA tasks and accuracy for other datasets. ${}^\dagger$Results are obtained based on our replication using default configurations. Other baseline results are sourced from \citet{shi2024dept}.}
\label{tab:mrqa and others}
\end{table*}

\section{Additional Results}
\subsection{Distribution of Pretrained Embeddings}\label{appendix: Distribution of Pretrained Embeddings}
\begin{figure}
    \centering
    \includegraphics[width=0.95\linewidth]{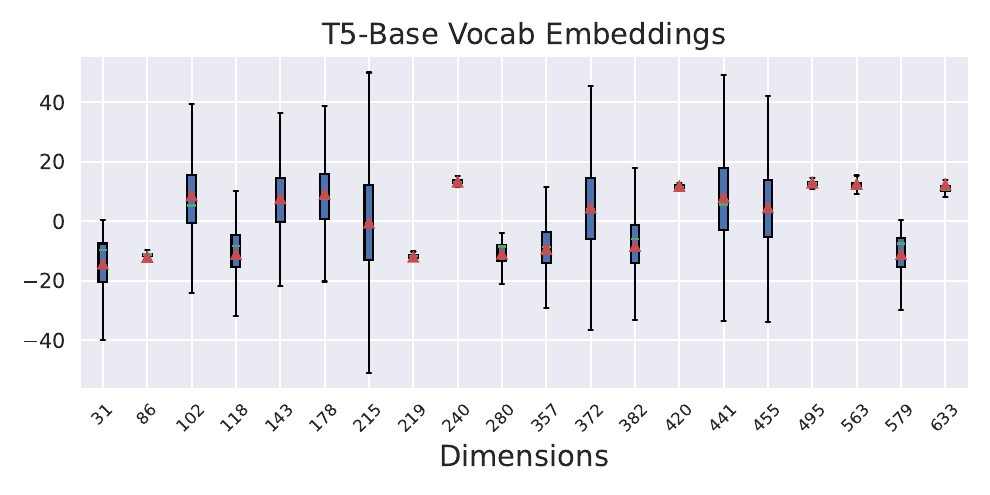}
    \caption{Distribution of randomly selected dimensions from the pretrained token embeddings of T5-base.}
    \label{fig:vocab embeddings distribution}
\end{figure}

Following our pilot study in Section~\ref{section: ulpt}, we present the distribution of the same randomly selected dimensions from the pretrained token embeddings of the T5-base model. We observe that the variation in these pretrained embeddings is smaller compared with the learned soft embeddings from the CoLA and SST-2 tasks.

\subsection{Additional Results on T5 Model}\label{appendix: mrqa results}
We evaluate ULPT on the MRQA dataset and four additional tasks in the ``Others'' category. Following the standard practice on these benchmarks~\cite{wang2023multitask, shi2024dept}, we run ULPT three times with different seeds and report the mean and standard deviation. This setup also allows us to assess ULPT’s sensitivity to the random initialization of the up-projection matrix.

The results are shown in Table~\ref{tab:mrqa and others}. Unlike GLUE and SuperGLUE performance, ULPT exhibits consistent improvement when the rank is higher. This is probably because these tasks are more challenging for the T5-base model, which aligns with the observation that full-model fine-tuning outperforms parameter-efficient methods on these tasks. Nevertheless, ULPT achieves competitive performance (slightly worse than the best-performing DePT approach), while saving parameters by multiple folds.

\begin{table*}[!t]
\centering
\resizebox{\textwidth}{!}{%
\begin{tabular}{lccccccccccccccccc}
\toprule
\textbf{Model} & \textbf{Method} & SST-2 & HQA & WG & \cellcolor{gray!20}\textbf{Avg.} & SST-2 & HQA & WG & \cellcolor{gray!20}\textbf{Avg.} & SST-2 & HQA & WG & \cellcolor{gray!20}\textbf{Avg.} & SST-2 & HQA & WG & \cellcolor{gray!20}\textbf{Avg.} \\
\midrule
 &  & \multicolumn{4}{c}{\#Param=2K, ULPT (r=2)} & \multicolumn{4}{c}{\#Param=4K, ULPT (r=16)} & \multicolumn{4}{c}{\#Param=8K, ULPT (r=64)} & \multicolumn{4}{c}{\#Param=28K, ULPT (r=256)} \\
\cmidrule(l){3-6} \cmidrule(l){7-10} \cmidrule(l){11-14} \cmidrule(l){15-18}
\multirow{2}{*}{Bloomz-560M} & PT & 89.8 & 42.9 & 48.6  & \cellcolor{gray!20}60.4 & 91.1 & 53.0 & 52.0 & \cellcolor{gray!20}65.4 & 91.9 & 57.2 & 50.0 & \cellcolor{gray!20}66.4 & 92.2 & 60.6 & 52.2  & \cellcolor{gray!20}68.3 \\
& ULPT & 90.2 & 52.4 & 51.7 & \cellcolor{gray!20}\textbf{64.8} & 92.2 & 55.7 & 53.1 & \cellcolor{gray!20}\textbf{67.0} & 91.8 & 59.3 & 53.1 & \cellcolor{gray!20}\textbf{68.1} & 92.6 & 62.5 & 52.2 & \cellcolor{gray!20}\textbf{69.1} \\
\midrule
 &  & \multicolumn{4}{c}{\#Param=4K, ULPT (r=2)} & \multicolumn{4}{c}{\#Param=6K, ULPT (r=16)} & \multicolumn{4}{c}{\#Param=10K, ULPT (r=64)} & \multicolumn{4}{c}{\#Param=30K, ULPT (r=256)} \\
\cmidrule(l){3-6} \cmidrule(l){7-10} \cmidrule(l){11-14} \cmidrule(l){15-18}
\multirow{2}{*}{Bloomz-1.7B} & PT & 93.2 & 64.6 & 50.1 & \cellcolor{gray!20}69.3  & 93.5 & 66.1 & 51.5 & \cellcolor{gray!20}70.4  & 94.0 & 67.3 & 55.3 & \cellcolor{gray!20}72.2 & 94.7 & 69.1 & 55.4  & \cellcolor{gray!20}73.1 \\
& ULPT & 94.4 & 65.6 & 54.6 & \cellcolor{gray!20}\textbf{71.5} & 93.9 & 66.3 & 55.6 & \cellcolor{gray!20}\textbf{71.9} & 94.3 & 68.0 & 55.2 & \cellcolor{gray!20}\textbf{72.5} & 95.1 & 69.3 & 57.4  & \cellcolor{gray!20}\textbf{73.9} \\
\midrule
 &  & \multicolumn{4}{c}{\#Param=5K, ULPT (r=2)} & \multicolumn{4}{c}{\#Param=8K, ULPT (r=16)} & \multicolumn{4}{c}{\#Param=13K, ULPT (r=64)} & \multicolumn{4}{c}{\#Param=31K, ULPT (r=256)} \\
\cmidrule(l){3-6} \cmidrule(l){7-10} \cmidrule(l){11-14} \cmidrule(l){15-18}
\multirow{2}{*}{Bloomz-3B} & PT & 93.2 & 66.1 & 50.5 & \cellcolor{gray!20}69.9 &  94.5 & 69.0 &56.0 & \cellcolor{gray!20}73.2 & 94.9 & 69.1 & 58.9 & \cellcolor{gray!20}74.3 & 94.9 & 71.5 & 60.0 & \cellcolor{gray!20}75.5 \\
& ULPT & 94.0 & 68.1 & 53.5 & \cellcolor{gray!20}\textbf{71.9}  & 94.4 & 68.9 & 57.1  &\cellcolor{gray!20} \textbf{73.5}  & 94.7 & 70.9  & 58.5 & \cellcolor{gray!20} \textbf{74.7} & 95.0 & 71.8 & 60.7 & \cellcolor{gray!20} \textbf{75.8}\\
\bottomrule
\end{tabular}}
\caption{Results on Bloomz, a decoder model with varying sizes (560M, 1.7B, and 3B) and hidden dimensions (1024, 2048, and 2560). We compare ULPT with prompt tuning by conditioning on the same number of trainable parameters.}
\label{tab: detailed decoder}
\end{table*}

\subsection{Additional Analysis on Bloomz Family}\label{appendix: bloomz}

We further test our hypothesis from Section~\ref{sec: analysis}: under a fixed parameter budget, longer low-dimensional prompts offer greater expressivity. To this end, we extend ULPT evaluation to Bloomz~\cite{muennighoff-etal-2023-crosslingual}, a decoder-only model available in 560M, 1.7B, and 3B sizes, having hidden dimensions of 1024, 2048, and 2560 respectively. 

For evaluation diversity, we select three mid-sized tasks from each task group: SST-2, HotpotQA, and Winogrande, providing assessment across classification, multi-hop reasoning, and commonsense reasoning. Given that Bloomz models are larger than the T5 series, we train up to 30K steps with a batch size of $4$, while keeping other hyperparamters the same as our main experiment.

We consider comparing ULPT with prompt tuning under different parameter budgets for text generation. Specifically, we vary the rank of ULPT from $2$ to $256$ while fixing the length $n=100$. For full-dimensional prompt tuning, the token number is adjusted to match the parameter count, resulting in fewer virtual tokens than ULPT. 

The results in Table~\ref{tab: detailed decoder} show that ULPT consistently outperforms full-dimensional prompt tuning across all model sizes and tasks given a fixed parameter budget. These results support our earlier analysis (Section~\ref{sec: analysis}) that allocating the parameter budget toward longer sequence lengths rather than higher embedding dimensions yields greater expressivity.

\subsection{Inference Overhead Analysis}\label{appendix: inference overhead}

\begin{table}[t]
\centering
\resizebox{0.5\textwidth}{!}{%
\begin{tabular}{lcc}
\toprule
\textbf{Runtime Setting}                  & \textbf{Llama 1B }     & \textbf{Llama 3B }      \\
\midrule
Vanilla PT embeddings loading  & $0.64 \pm 0.04$    & $0.91 \pm 0.04$    \\
\hdashline                                
ULPT up-projection ($r=2$)       & $0.56 \pm 0.06$    & $0.59 \pm 0.04$    \\
ULPT up-projection ($r=64$)      & $1.43 \pm 0.09$    & $1.87 \pm 0.06$    \\
ULPT up-projection ($r=256$)     & $4.09 \pm 0.10$    & $5.80 \pm 0.16$    \\
\hdashline
Decoding                         & $1481.15 \pm 64.26$ & $2536.67 \pm 42.14$ \\
\bottomrule
\end{tabular}}
\caption{Runtime (ms) for ULPT prompt reconstruction.}
\label{tab:runtime}
\end{table}

\begin{table}[t]
\centering
\resizebox{0.38\textwidth}{!}{%
\begin{tabular}{lcc}
\toprule
\textbf{Model}       & \textbf{No adaptation}    & \textbf{ULPT}   \\
\midrule
Llama 1B    & $82.76 \pm 0.33$    & $82.71 \pm 0.33$   \\
Llama 3B    & $48.74 \pm 0.25$    & $48.70 \pm 0.22$   \\
\bottomrule
\end{tabular}}
\caption{Decoding speed (tokens/second)}
\label{tab:decode-speed}
\end{table}

ULPT's ideal use case is for massive LLM customization, where adapters are frequently loaded to a model. Unlike vanilla prompt tuning, ULPT reconstructs prompt embeddings from a low-dimensional space at inference time.
We empirically compare its up-projection runtime against vanilla PT over 100 runs, shown in Table~\ref{tab:runtime}. We find the overhead is negligible relative to the decoding time on GSM8K .

ULPT follows the paradigm of prompt tuning to prepend tokens. We also measure end-to-end decoding speed using ULPT (rank 2 and 100 prompt tokens) on a single NVIDIA A6000 48G GPU with the Huggingface generation API. As seen in Table~\ref{tab:decode-speed}, the decoding speed is similar with or without the additional prompt tokens.

\onecolumn
\input{theory_apx}

\end{document}

%% file: math_commands.tex
\usepackage{amsmath,amsfonts,amsthm,bm,physics}
\def\vb{{\bm{b}}}

\def\ve{{\bm{e}}}

\def\vs{{\bm{s}}}

\def\vx{{\bm{x}}}
\def\vy{{\bm{y}}}
\def\vz{{\bm{z}}}

\def\mA{{\bm{A}}}

\def\mE{{\bm{E}}}

\def\mP{{\bm{P}}}

\def\mZ{{\bm{Z}}}

\DeclareMathAlphabet{\mathsfit}{\encodingdefault}{\sfdefault}{m}{sl}
\SetMathAlphabet{\mathsfit}{bold}{\encodingdefault}{\sfdefault}{bx}{n}

\def\gL{{\mathcal{L}}}

\def\sN{{\mathbb{N}}}

\def\sR{{\mathbb{R}}}

\DeclareMathOperator{\diag}{diag}

\theoremstyle{plain}
\newtheorem{theorem}{Theorem}

\newtheorem{lemma}[theorem]{Lemma}

\theoremstyle{definition}

\newtheorem{assumption}[theorem]{Assumption}

\theoremstyle{remark}

%% file: theory.tex
We first show that an ultra low-dimensional space can capture the structure of the original embeddings (i.e., expressiveness). We then show the convergence of gradient descent with our random projection (i.e., optimization).

\paragraph{Expressiveness.}

Our low-dimensional parameterization approximately captures high-dimensional structure with high confidence. To show this, we first state the following lemma.

\begin{lemma}\label{lem:im}
      Sample a random matrix $\mA \in \sR^{r \times m}$ such that each element follows the standard Gaussian distribution. Let $\epsilon \in (0, 1/2]$ and $r \in \sN_+$. There exists a constant $c$ such that \begin{align}
        \Pr\left( \left| \frac{(1/{\sqrt{r}}) \|  \mA\vx \| - \|\vx \|}{\|\vx\|} \right| \ge \epsilon \right) \le \frac{2}{ \exp\left(\epsilon^2r\right/c)}
    \end{align}
    for any $\vx \in \sR^d$. 
\end{lemma}

This result is adapted from \citet{indyk1998approximate}. Essentially, the lemma characterizes the high-probability bound of the well known Johnson–Lindenstrauss lemma~\cite{dasgupta2003elementary, matouvsek2008variants}. 
Based on this, we formally show the expressiveness of our ultra low-dimensional embeddings in the following theorem.

\begin{restatable}{theorem}{expr}\label{thm:existence}
Let $\ve_1,\dots, \ve_n \in \mathbb{R}^d$ be the embedding vectors in a high-dimensional space. Let $\mP \in \mathbb{R}^{r \times d}$ be a random projection matrix where each element $p_{i,j} \sim \mathcal{N}(0,1/r)$, and let $\vz_i = \mP\ve_i \in \mathbb{R}^r$ be the projected low-dimensional vectors.

Then, with confidence at least $1-\delta$, we have
\begin{align}
    \resizebox{1.02\columnwidth}{!}{%
    $(1-\epsilon) \| \ve_i - \ve_j \| \le \| \vz_i - \vz_j \| \le (1+\epsilon) \| \ve_i - \ve_j \|$
    }
\end{align}
for all $i,j \in [n]$, given that $r \ge C \epsilon^{-2} \log (2n/\delta)$ for a sufficiently large constant $C$.
\end{restatable}

\begin{proof}
See Appendix~\ref{apx:proof:expr}.
\end{proof}

In essence, our theorem asserts that by projecting data using a random matrix $\mP$, the pairwise $L^2$ distances of the original high-dimensional vectors are preserved for all $(i, j)$ pairs with high probability. Crucially, the projected dimension $r$ scales logarithmically with the number of embeddings $n$, rather than the original dimension $d$, demonstrating a favorable property of scaling. 

It should be noted that, although our theorem uses $L^2$ as the metric, it can be extended to the dot-product metric by $\|\vx - \vy\|^2 = \|\vx\|^2 + \|\vy\|^2 - 2 \vx \cdot \vy$. Practically, since LLMs compute attention using pairwise dot products between embeddings~\cite{NIPS2017_3f5ee243}, Theorem~\ref{thm:existence} implies that our up-projected low-dimensional prompts preserve the relational structure of full-dimensional embeddings.

\paragraph{Optimization.} The above theorem shows the expressiveness of the low-dimensional space. We assert in the following theorem that, given a random up-projection matrix, the optimal low-dimensional embeddings can be learned by gradient descent under mild assumptions.

\begin{restatable}{theorem}{optim}\label{thm:optim}
Assume the original loss function $\gL$  is Polyak--Lojasiewic and element-wise Lipschitz on the original $d$-dimensional embeddings. Let $\mP \in \sR^{r\times d}$ be a given full-rank  random Gaussian matrix (i.e., rank $r$), and our parametrization be $\hat\ve_i = \diag(\vs) \mP^\top \vz_i + \vb$. With a proper learning rate schedule $\eta_1, \eta_2, \dots$, our parameters $\vx = [\vb, \vs, \vz_1, \dots, \vz_n]$ converge to the global optimum with gradient descent if $\vs$ is always non zero.
\end{restatable}
\begin{proof}
    See Appendix~\ref{apx:proof:optim}.
\end{proof}

Theorem~\ref{thm:optim} shows that, even with the na\"ive gradient descent, the fixed random matrix $\mP$ does not hinder the optimization procedure. By combining Theorem~\ref{thm:existence}, we theoretically justify our overall practice of ULPT.

%% file: theory_apx.tex
\allowdisplaybreaks
\section{Theoretical Analysis}\label{apx:proof}

\subsection{Proof of Theorem~\ref{thm:existence}}\label{apx:proof:expr}
\expr*
\begin{proof}

Setting $\vz_i = \mP \ve_i$, we have
\begin{align}
        \Pr\left( \left| \frac{ \| \vz_i - \vz_j \| - \|\ve_i - \ve_j\|}{\|\ve_i - \ve_j\|} \right| \ge \epsilon \right)
        =& \Pr\left( \left| \frac{ \| \mP \left( \ve_i - \ve_j \right) \| - \|\ve_i - \ve_j\|}{\|\ve_i - \ve_j\|}  \right| \ge \epsilon \right) \\
        \le& \frac{2}{ \exp\left(\epsilon^2r/c\right)},
    \end{align}
    for any $i,j \in [n]$. The last inequality is a direction application of Lemma~\ref{lem:im}. Further, Boole's inequality suggests
    \begin{align}
        \Pr\left(\text{any } i,j \in [n] : \left| \frac{ \| \vz_i - \vz_j \| - \|\ve_i - \ve_j\|}{\|\ve_i - \ve_j\|} \right| \ge \epsilon \right) \le n^2 \frac{2}{ \exp\left(\epsilon^2r/c\right)},\label{eq:hp}
    \end{align}
    where $n^2$ comes from counting all $(i,j)$ pairs. By setting $\delta > 0$ to any value smaller than $\frac{2n^2}{ \exp\left(\epsilon^2r/c\right)}$, we have $r \ge 2c\epsilon^{-2} \cdot \log(2n/\delta)$. Therefore, Eqn.~\eqref{eq:hp} can be rewritten as follows: with confidence at least $1-\delta$, we have \begin{align}
        (1-\epsilon) \| \ve_i - \ve_j \| \le \| \vz_i - \vz_j \| \le (1+\epsilon) \| \ve_i - \ve_j \|
    \end{align}
    for all $i,j \in [n]$, as long as $r\ge 2c\epsilon^{-2}  \log(2n/\delta)$.
\end{proof}

\subsection{Proof of 
Theorem~\ref{thm:optim}}\label{apx:proof:optim}

We first formally explain our assumptions.

\begin{assumption}
The loss function $\gL$ is $\beta$ element-wise Lipschitz w.r.t.~embeddings. Specifically, we have \begin{align}
     | \nabla \gL(x_i) - \nabla \gL(y_i) | \le \beta | x_i - y_i | 
\end{align}
for any $\vx, \vy \in \sR^{nd}$ being unrolled from $n \times d$ embedding matrices. $x_i$ and $y_i$ are elements in the vectors.
\end{assumption}

\begin{assumption}
The loss function $\gL$ is  $\mu$-PL (Polyak--Lojasiewic) w.r.t.~embeddings, meaning that \begin{align}
   \frac{1}{2} \| \nabla \gL(\vx) \|_2^2 \ge \mu\left( \gL(\vx) - \gL(\vx^*)\right)
\end{align}
for any $\vx \in \sR^{nd}$, where $\vx$ is embedding parameters and $\vx^*$ is any finite minimizer of $\gL$.
\end{assumption}

These are the common assumptions used to show the optimization process in deep learning~\citep{karimi2016linear,mei2020global}. In addition, we also impose an assumptions on the projection matrix and the scaling vector $\vs$.

\begin{assumption}
    The projection matrices $\mP \in \sR^{r\times d}$ has a rank of $r$. In addition, we assume $\vs$ is not a zero vector during optimization.
\end{assumption}

Based on these assumptions, we first provide the essential lemmas for our proof.

\begin{lemma}\label{lem:element-to-global}
    If $\gL: \sR^d \to \sR$ is $\beta$-Lipschitz in each element, then $\gL$ is $\beta$-Lipschitz.
\end{lemma}
\begin{proof}
    Let $\nabla \gL(x_i)$ be the partial derivative of $\mathcal L$ w.r.t.~$x_i$. We have \begin{align}
        | \nabla \gL(x_i) - \nabla \gL(y_i) | \le \beta | x_i - y_i|
    \end{align}
    for every $x_i,y_i\in\mathbb R$.
    Therefore, \begin{align}
        \| \nabla \gL(\vx) - \nabla \gL(\vy) \|^2
        =& \sum_{i=1}^d \left|\nabla \gL(x_i) - \nabla \gL(y_i)\right|^2 \\
        \le& \sum_{i=1}^d \beta^2 | x_i - y_i|^2 \\
        =& \beta^2 \|\vx - \vy\|^2.
    \end{align}
    We complete the proof by taking the square root on both sides.
\end{proof}

\begin{lemma}\label{lem:lip}
    Let $\hat\gL(\hat \vx)$ be the loss function with our ULPT approach, where $\hat\vx\in\mathbb R^{nr+2d}$ is the concatenation of shift/scale embeddings and the ultra-low-dimensional prompt embeddings.  $\hat\gL(\hat \vx)$ is $\beta'$-Lipschitz w.r.t.~$\hat\vx$ for some $\beta'>0$. 
\end{lemma}
\begin{proof} We prove the Lipschitz condition of $\mathcal L$ w.r.t.~the ultra-low-dimensional prompt embeddings, scale embedding, and shift embedding separately. Then, Lemma~\ref{lem:element-to-global} suggests the Lipschitz condition of $\mathcal L$ w.r.t to $\hat {\bm x}$. Without loss of generality, we assume the layout of parameters is $\hat {\bm x}=[\vb, \vs, \vz_1, \vz_2, \dots, \vz_n]$, where $n$ is the number of prompt tokens. 

We first calculate partial derivatives as follows \begin{align}
       & \frac{\partial \gL}{\partial \vb}
        = \sum_{i=1}^n\left(\frac{\partial \hat\ve_i}{\partial \vb}\right)^\top \frac{\partial \gL}{\partial \hat \ve_i}  = \sum_{i=1}^n \frac{\partial \gL}{\partial \hat\ve_i}, \label{eq:db}\\
        & \frac{\partial \gL}{\partial \vs} = \left(\frac{\partial \hat\ve_i}{\partial \vs}\right)^\top \frac{\partial \gL}{\partial \hat \ve_i}  = \sum_{i=1}^n \diag( \mP^\top \vz_i)  \frac{\partial \gL}{\partial \hat\ve_i},  \text{and } \label{eq:ds}\\
        & \frac{\partial \gL}{\partial \vz_i}        = \left(\frac{\partial \hat\ve_i}{\partial \vz_i}\right)^\top \frac{\partial \gL}{\partial \hat \ve_i}  = \mP \diag(\vs) \frac{\partial \gL}{\partial \hat\ve_i}. \label{eq:dqi}
    \end{align}
    % This suggests that the gradient of $\vz_i$ is the random down projection of the original gradient $\frac{\partial \gL}{\partial \hat\ve_i}$.Without loss of generality, we assume the layout of parameters is $[\vb, \vs, \vz_1, \vz_2, \dots, \vz_n] \in \sR^{2d + nr}$. 

 Our proof of the Lipschitz condition starts with checking $\vb$. For any element $b_k$, where $k=1,\cdots, d$, we have \begin{align}
        \left|\nabla \hat\gL(b_k^{(1)}) - \nabla \hat\gL(b_k^{(2)} )  \right| =&  
        % \sum_{i=1}^n \frac{\partial \gL}{\partial \hat\ve_i} \\
        \left| \sum_{i=1}^n \left( \nabla \gL(\hat e_{i,k}^{(1)}) - \nabla \gL(\hat e_{i,k}^{(2)}) \right) \right| \\
        \le&  \sum_{i=1}^n \left|  \nabla \gL(\hat e_{i,k}^{(1)}) - \nabla \gL(\hat e_{i,k}^{(2)}) \right| \\
        \le&  L \sum_{i=1}^n |\hat e_{i,k}^{(1)} - \hat e_{i,k}^{(2)}| \\
        =& nL | b_k^{(1)} - b_k^{(2)} |
    \end{align}
    where superscripts (1) and (2) indicate two values in the Lipschitz condition. $\hat e_{i,k}$ refers to the $i$th prompt token and its $k$th dimension. Here, the first equation is due to Eqn.~\eqref{eq:db}.

 For the scale embedding $\vs$, we also consider the $k$th dimension for $k=1,\cdots, d$: \begin{align}
        \left|\nabla \hat\gL(s_k^{(1)}) - \nabla \hat\gL(s_k^{(2)})\right|
        =& \left| \sum_i \left(\vz_i^\top \mP_{:,k} \nabla \gL(\hat e_{i,k}^{(1)}) - \vz_i^\top\mP_{:,k}  \nabla \gL(\hat e_{i,k}^{(2)})  \right) \right| \\
        =&  \left|\sum_i \left( \vz_i^\top \mP_{:,k} \right) \left(\nabla \gL(\hat e_{i,k}^{(1)}) - \nabla \gL(\hat e_{i,k}^{(2)})  \right)\right| \\
        \le&  \sqrt{\sum_i \left( \vz_i^\top \mP_{:,k} \right)^2} \sqrt{\sum_i \left(\nabla \gL(\hat e_{i,k}^{(1)}) - \nabla \gL(\hat e_{i,k}^{(2)}) \right)^2 } \label{eq:cs}\\
        % \le&  \sum_i \left| \left( \vz_i^\top \mP_{:,k} \right) \left(\nabla \gL^{s_k}(x) - \nabla \gL^{s_k}(y)  \right)\right| \\
        % \le&  \sum_i \left| \vz_i^\top \mP_{:,k} \right| \left|\nabla \gL^{s_k}(x) - \nabla \gL^{s_k}(y) \right| \\
        \le& \sum_i\|\vz_i\| \|\mP_{:,k}\| L \sqrt{\sum_i \left(\hat e_{i,k}^{(1)} - \hat e_{i,k}^{(2)}\right)^2} \\
        \le& L n  \sigma_{\max}(\mZ) \sigma_{\max}(\mP) \sqrt{\sum_i  (\vz_i^\top \mP_{:,k})^2 (\hat s_{k}^{(1)} - \hat s_{k}^{(2)})^2 } 
        \label{eq:norm}\\
        \le& L n  \sigma_{\max}(\mZ) \sigma_{\max}(\mP) \sqrt{\sum_i  (\vz_i^\top \mP_{:,k})^2} |\hat s_{k}^{(1)} - \hat s_{k}^{(2)}|  \\
        \le&  L n \sigma_{\max}^2(\mZ) \sigma_{\max}^2(\mP) \left| \hat s_{k}^{(1)} - \hat s_{k}^{(2)}\right|,
    \end{align}
where $\mP_{:,k}$ is the $k$th column of the $\mP$ matrix (as a column vector), and $\sigma_{\max}(\cdot)$ is the maximum singular value of a matrix. Here, Line~\eqref{eq:cs} is obtained by applying the  Cauchy--Schwartz inequality. Line~\eqref{eq:norm} is based on matrix norm inequalities. 
    
    Finally, we examine $z_{i,k}$, which is the $k$th dimension ($k=1,\cdots, r$) of the $i$th token of our ultra-low-dimensional embeddings: \begin{align}
        \left|\nabla \hat\gL(z_{i,k}^{(1)}) - \nabla \hat\gL(z_{i,k}^{(2)})\right|
        =& \left| \mP_{k,:} \diag(\vs) 
        \frac{\partial \gL}{\partial \hat \ve_i^{(1)}}
         - \mP_{k,:} \diag(\vs) \frac{\partial \gL}{\partial \hat \ve_i^{(2)}} \right| \\
        =& \left| \sum_j p_{k,j} s_j \left(\nabla \gL(\hat e_{ij}^{(1)}) -  \nabla \gL(\hat e_{ij}^{(2)})  \right)\right| \\
        \le& \left\| \mP_{k,:} \diag(\vs) \right\| \left\| \frac{\partial \gL}{\partial \hat \ve_i^{(1)}} -  \frac{\partial \gL}{\partial \hat \ve_i^{(2)}} \right\| \\
        \le& \sigma_{\max}(\mP) \sigma_{\max}(\vs) L \| \hat\ve_i^{(1)} - \hat\ve_i^{(2)} \| \\
        \le&  \sigma_{\max}(\mP) \sigma_{\max}(\vs) L \left\| \diag(\vs) \mP^\top \left(\vz_i^{(1)} - \vz_i^{(2)}\right) \right\| \\
        \le& L  \sigma_{\max}^2(\mP) \sigma_{\max}^2(\vs)  \| \vz_i^{(1)} - \vz_i^{(2)}\| \\
        =& L  \sigma_{\max}^2(\mP) \sigma_{\max}^2(\vs)  |z_{i,k}^{(1)} - z_{i,k}^{(2)}|. \label{eq:last}
    \end{align}
where Eqn.~\eqref{eq:last} holds because we examine one element $z_{i,k}$ at a time, so $z_{i,k'}^{(1)}=z_{i,k'}^{(2)}$ for $k'\ne k$.

    With these element-wise properties, we can have the full-parameter Lipschitz condition by using Lemma~\ref{lem:element-to-global}.
\end{proof}

\begin{lemma}\label{lem:pl}
    The loss function $\hat \gL$ is  $\mu'$-PL (Polyak--Lojasiewic) w.r.t.~$\hat\vx \in {\sR^{nr+d}}$ for some $\mu'$.
\end{lemma}
\begin{proof}
\begin{align}
     \frac{1}{2} \| \nabla \hat\gL(\hat \vx) \|^2
     =& \frac{1}{2}  \left\| \frac{\partial \gL}{\partial \vb} \right\|^2 + \frac{1}{2}  \left\| \frac{\partial \gL}{\partial \vs} \right\|^2  + \frac{1}{2}  \sum_{i=1}^n \left\|\frac{\partial \gL}{\partial \vz_i} \right\|^2 \\
     =& \frac{1}{2}  \left\| \sum_{i=1}^n \frac{\partial \gL}{\partial \hat \ve_i} \right\|^2 + \frac{1}{2}  \left\| \sum_{i=1}^n \diag( \mP^\top \vz_i) \frac{\partial \gL}{\partial \hat \ve_i} \right\|^2 + \frac{1}{2}  \sum_{i=1}^n \left\| \mP \diag(\vs) \frac{\partial \gL}{\partial \hat \ve_i} \right\|^2 \\
     \ge& \frac{1}{2}  \sum_{i=1}^n \left\| \mP \diag(\vs) \frac{\partial \gL}{\partial \hat \ve_i} \right\|^2 \\
    \ge& \frac{1}{2}  \sigma_{\min}^2(\mP)\sigma_{\min}^2(\vs) \sum_{i=1}^n \left\| \frac{\partial \gL}{\partial \hat \ve_i} \right\|^2 \\
    =& \frac{1}{2}  \sigma_{\min}^2(\mP)\sigma_{\min}^2(\vs) \| \nabla \gL(\hat\vx) \|^2 \\
    \ge& \sigma_{\min}^2(\mP)\sigma_{\min}^2(\vs)\mu\left( \gL(\hat\vx) - \gL(\vx^*)\right) \\
    \ge& \sigma_{\min}^2(\mP)\sigma_{\min}^2(\vs)\mu\left( \gL(\hat\vx) - \gL(\hat\vx^*)\right),
\end{align}
where $\hat \vx^*$ is the minimizer under our parameterization. This suggests that $\gL$ is $\mu'$-PL for some $\mu'$.
\end{proof}

\optim*
\begin{proof}
At each iteration $t$, gradient descent produces \begin{align}
    \vx_{t+1} \gets \vx_{t} - \eta_t \nabla  \gL(\vx_t),
\end{align}
where $\gL$ is the loss function under our parametrization. 
For each iteration, we choose $\eta_t = 1/\beta'(\vx_t)$, where $\beta'(\vx_t)$ is the Lipschitz coefficient in Lemma~\ref{lem:lip} depending on $\vx_t$: \begin{align}
    \gL(\vx_{t+1}) \le& \gL(\vx_t) +  \left(\nabla \gL(\vx_t) \right)^\top (\vx_{t+1} - \vx_t) + \frac{\beta'(\vx_t)}{2} \| \vx_{t+1} - \vx_{t} \|^2 \\
    =& \gL(\vx_t) - \frac{1}{2\beta'(\vx_t)} \|\nabla \gL(\vx_t) \|^2 \\
    \le& \gL(\vx_t) - \frac{\mu'(\vx_t)}{\beta'(\vx_t)}(\gL(\vx_t) - \gL(\vx^*)).
\end{align}
where $\mu'(\vx_t)$ is the PL coefficient in Lemma~\ref{lem:pl}, which also depends on $\vx_t$.
By rearranging the terms, we obtain\begin{align}
        \gL(\vx_{t+1}) - \gL(\vx^*) \le \left(1 - \frac{\mu'(\vx_t)}{\beta'(\vx_t)}\right) ( \gL(\vx_t) - \gL(\vx^*)),
    \end{align}
    suggesting that the excessive loss $\gL(\vx) - \gL(\vx^*)$ converges to 0.
\end{proof}
Note that our Lipschitz and PL conditions are non-uniform (i.e., depending on the parameters according to the lemmas above). Therefore, a proper learning schedule $\eta_t = 1/\beta(\vx_t)$ is needed in the theoretical analysis.

%% file: custom.bib
@inproceedings{asai-etal-2022-attempt,
    title = "{ATTEMPT}: Parameter-Efficient Multi-task Tuning via Attentional Mixtures of Soft Prompts",
    author = "Asai, Akari  and
      Salehi, Mohammadreza  and
      Peters, Matthew  and
      Hajishirzi, Hannaneh",
    booktitle = "Proceedings of the 2022 Conference on Empirical Methods in Natural Language Processing",
    year = "2022",
    url = "https://aclanthology.org/2022.emnlp-main.446",
    pages = "6655--6672",
}

@InProceedings{karimi2016linear,
author="Karimi, Hamed
and Nutini, Julie
and Schmidt, Mark",
title="Linear Convergence of Gradient and Proximal-Gradient Methods Under the {P}olyak-{L}ojasiewicz Condition",
booktitle="Machine Learning and Knowledge Discovery in Databases",
year="2016",
pages="795--811",
url="https://link.springer.com/chapter/10.1007/978-3-319-46128-1_50"
}

@inproceedings{mei2020global,
  title     = {On the Global Convergence Rates of {S}oftmax Policy Gradient Methods},
  author    = {Mei, Jincheng and Xiao, Chenjun and Szepesvari, Csaba and Schuurmans, Dale},
  booktitle = {Proceedings of the 37th International Conference on Machine Learning},
  pages     = {6820-6829},
  year      = {2020},
  url       = {https://proceedings.mlr.press/v119/mei20b.html}
}

@inproceedings{indyk1998approximate,
  title     = {Approximate nearest neighbors: {T}owards removing the curse of dimensionality},
  author    = {Indyk, Piotr and Motwani, Rajeev},
  booktitle = {Proceedings of the Thirtieth Annual ACM Symposium on Theory of Computing},
  pages     = {604--613},
  year      = {1998},
  url       = {https://dl.acm.org/doi/10.1145/276698.276876}
}

@article{dasgupta2003elementary,
  title   = {An elementary proof of a theorem of {J}ohnson and {L}indenstrauss},
  author  = {Dasgupta, Sanjoy and Gupta, Anupam},
  journal = {Random Structures \& Algorithms},
  volume  = {22},
  number  = {1},
  pages   = {60--65},
  year    = {2003},
  url     = {https://doi.org/10.1002/rsa.10073}
}

@article{matouvsek2008variants,
  title   = {On variants of the {J}ohnson--{L}indenstrauss lemma},
  author  = {Matou{\v{s}}ek, Ji{\v{r}}{\'\i}},
  journal = {Random Structures \& Algorithms},
  volume  = {33},
  number  = {2},
  pages   = {142--156},
  year    = {2008},
  url     = {https://doi.org/10.1002/rsa.20218}
}

@inproceedings{wang2023multitask,
title={Multitask Prompt Tuning Enables Parameter-Efficient Transfer Learning},
author={Zhen Wang and Rameswar Panda and Leonid Karlinsky and Rogerio Feris and Huan Sun and Yoon Kim},
booktitle={The Eleventh International Conference on Learning Representations},
year={2023},
url={https://openreview.net/forum?id=Nk2pDtuhTq}
}

@inproceedings{shi2024dept,
title={De{PT}: Decomposed Prompt Tuning for Parameter-Efficient Fine-tuning},
author={Zhengxiang Shi and Aldo Lipani},
booktitle={International Conference on Learning Representations},
year={2024},
url={https://openreview.net/forum?id=KjegfPGRde}
}

@inproceedings{
wu2024zeroshot,
title={Zero-Shot Continuous Prompt Transfer: Generalizing Task Semantics Across Language Models},
author={Zijun Wu and Yongkang Wu and Lili Mou},
booktitle={International Conference on Learning Representations},
year={2024},
url={https://openreview.net/forum?id=26XphugOcS}
}

@inproceedings{xiao-etal-2023-decomposed,
    title = "Decomposed Prompt Tuning via Low-Rank Reparameterization",
    author = "Xiao, Yao  and
      Xu, Lu  and
      Li, Jiaxi  and
      Lu, Wei  and
      Li, Xiaoli",
    booktitle = "Findings of the Association for Computational Linguistics: EMNLP",
    year = "2023",
    url = "https://aclanthology.org/2023.findings-emnlp.890",
    pages = "13335--13347",
}

@article{raffel2020exploring,
  title={Exploring the Limits of Transfer Learning with a Unified Text-to-Text Transformer},
  author={Raffel, Colin and Shazeer, Noam and Roberts, Adam and Lee, Katherine and Narang, Sharan and Matena, Michael and Zhou, Yanqi and Li, Wei and Liu, Peter J},
  journal={Journal of Machine Learning Research},
  pages={1--67},
  year={2020}, 
  url={https://jmlr.org/papers/v21/20-074.html}
}

@inproceedings{wei2022finetuned,
title={Finetuned Language Models are Zero-Shot Learners},
author={Jason Wei and Maarten Bosma and Vincent Zhao and Kelvin Guu and Adams Wei Yu and Brian Lester and Nan Du and Andrew M. Dai and Quoc V Le},
booktitle={International Conference on Learning Representations},
year={2022},
url={https://openreview.net/forum?id=gEZrGCozdqR}
}

@inproceedings{hu2022lora,
title={Lo{RA}: Low-Rank Adaptation of Large Language Models},
author={Edward J Hu and yelong shen and Phillip Wallis and Zeyuan Allen-Zhu and Yuanzhi Li and Shean Wang and Lu Wang and Weizhu Chen},
booktitle={International Conference on Learning Representations},
year={2022},
url={https://openreview.net/forum?id=nZeVKeeFYf9}
}

@inproceedings{lester-etal-2021-power,
    title = "The Power of Scale for Parameter-Efficient Prompt Tuning",
    author = "Lester, Brian  and
      Al-Rfou, Rami  and
      Constant, Noah",
    booktitle = "Proceedings of the 2021 Conference on Empirical Methods in Natural Language Processing",
    year = "2021",
    url = "https://aclanthology.org/2021.emnlp-main.243",
    pages = "3045--3059",
}

@inproceedings{li-liang-2021-prefix,
    title = "Prefix-{T}uning: Optimizing Continuous Prompts for Generation",
    author = "Li, Xiang Lisa  and
      Liang, Percy",
    booktitle = "Proceedings of the 59th Annual Meeting of the Association for Computational Linguistics and the 11th International Joint Conference on Natural Language Processing",
    year = "2021",
    url = "https://aclanthology.org/2021.acl-long.353",
    pages = "4582--4597",
}

@article{LIU2024208,
title = {{GPT} understands, too},
journal = {AI Open},
pages = {208-215},
year = {2024},
url = {https://www.sciencedirect.com/science/article/pii/S2666651023000141},
author = {Xiao Liu and Yanan Zheng and Zhengxiao Du and Ming Ding and Yujie Qian and Zhilin Yang and Jie Tang},
}

@inproceedings{petrov2024when,
title={When Do Prompting and Prefix-Tuning Work? {A} Theory of Capabilities and Limitations},
author={Aleksandar Petrov and Philip Torr and Adel Bibi},
booktitle={International Conference on Learning Representations},
year={2024},
url={https://openreview.net/forum?id=JewzobRhay}
}

@inproceedings{petrov2024prompting,
title={Prompting a Pretrained Transformer Can Be a Universal Approximator},
author={Aleksandar Petrov and Philip Torr and Adel Bibi},
booktitle={Proceedings of the 41st International Conference on Machine Learning},
year={2024},
url={https://proceedings.mlr.press/v235/petrov24a.html},
pages = {40523--40550},
}

@inproceedings{NEURIPS2020_1457c0d6,
 author = {Brown, Tom and Mann, Benjamin and Ryder, Nick and Subbiah, Melanie and Kaplan, Jared D and Dhariwal, Prafulla and Neelakantan, Arvind and Shyam, Pranav and Sastry, Girish and Askell, Amanda and Agarwal, Sandhini and Herbert-Voss, Ariel and Krueger, Gretchen and Henighan, Tom and Child, Rewon and Ramesh, Aditya and Ziegler, Daniel and Wu, Jeffrey and Winter, Clemens and Hesse, Chris and Chen, Mark and Sigler, Eric and Litwin, Mateusz and Gray, Scott and Chess, Benjamin and Clark, Jack and Berner, Christopher and McCandlish, Sam and Radford, Alec and Sutskever, Ilya and Amodei, Dario},
 booktitle = {Advances in Neural Information Processing Systems},
 pages = {1877--1901},
 title = {Language Models are Few-Shot Learners},
 url = {https://proceedings.neurips.cc/paper_files/paper/2020/file/1457c0d6bfcb4967418bfb8ac142f64a-Paper.pdf},
 year = {2020}
}

@inproceedings{NEURIPS2022_8bb0d291,
 author = {Kojima, Takeshi and Gu, Shixiang (Shane) and Reid, Machel and Matsuo, Yutaka and Iwasawa, Yusuke},
 booktitle = {Advances in Neural Information Processing Systems},
 pages = {22199--22213},
 title = {Large Language Models are Zero-Shot Reasoners},
 url = {https://openreview.net/pdf?id=e2TBb5y0yFf},
 year = {2022}
}

@inproceedings{dong-etal-2024-survey,
    title = "A Survey on In-context Learning",
    author = "Dong, Qingxiu  and
      Li, Lei  and
      Dai, Damai  and
      Zheng, Ce  and
      Ma, Jingyuan  and
      Li, Rui  and
      Xia, Heming  and
      Xu, Jingjing  and
      Wu, Zhiyong  and
      Chang, Baobao  and
      Sun, Xu  and
      Li, Lei  and
      Sui, Zhifang",
    booktitle = "Proceedings of the 2024 Conference on Empirical Methods in Natural Language Processing",
    year = "2024",
    url = "https://aclanthology.org/2024.emnlp-main.64",
    pages = "1107--1128",
}

@inproceedings{NEURIPS2022_9d560961,
 author = {Wei, Jason and Wang, Xuezhi and Schuurmans, Dale and Bosma, Maarten and ichter, brian and Xia, Fei and Chi, Ed and Le, Quoc V and Zhou, Denny},
 booktitle = {Advances in Neural Information Processing Systems},
 pages = {24824--24837},
 title = {Chain-of-Thought Prompting Elicits Reasoning in Large Language Models},
 url = {https://openreview.net/pdf?id=_VjQlMeSB_J},
 year = {2022}
}

@article{touvron2023llama,
  title={Llama: Open and efficient foundation language models},
  author={AI @ Meta},
  journal={arXiv preprint arXiv:2302.13971},
  url={https://arxiv.org/abs/2302.13971},
  year={2023}
}

@inproceedings{aghajanyan-etal-2021-intrinsic,
    title = "Intrinsic Dimensionality Explains the Effectiveness of Language Model Fine-Tuning",
    author = "Aghajanyan, Armen  and
      Gupta, Sonal  and
      Zettlemoyer, Luke",
    booktitle = "Proceedings of the 59th Annual Meeting of the Association for Computational Linguistics and the 11th International Joint Conference on Natural Language Processing",
    year = "2021",
    url = "https://aclanthology.org/2021.acl-long.568",
    pages = "7319--7328",
}

@inproceedings{liu-etal-2022-p,
    title = "P-{T}uning: Prompt Tuning Can Be Comparable to Fine-tuning Across Scales and Tasks",
    author = "Liu, Xiao  and
      Ji, Kaixuan  and
      Fu, Yicheng  and
      Tam, Weng  and
      Du, Zhengxiao  and
      Yang, Zhilin  and
      Tang, Jie",
    booktitle = "Proceedings of the 60th Annual Meeting of the Association for Computational Linguistics",
    year = "2022",
    url = "https://aclanthology.org/2022.acl-short.8",
    pages = "61--68",
}

@inproceedings{choi-etal-2023-smop,
    title = "{SM}o{P}: Towards Efficient and Effective Prompt Tuning with Sparse Mixture-of-Prompts",
    author = "Choi, Joon-Young  and
      Kim, Junho  and
      Park, Jun-Hyung  and
      Mok, Wing-Lam  and
      Lee, SangKeun",
    booktitle = "Proceedings of the 2023 Conference on Empirical Methods in Natural Language Processing",
    year = "2023",
    url = "https://aclanthology.org/2023.emnlp-main.884",
    pages = "14306--14316",
}

@inproceedings{razdaibiedina-etal-2023-residual,
    title = "Residual {P}rompt {T}uning: {I}mproving prompt tuning with residual reparameterization",
    author = "Razdaibiedina, Anastasiia  and
      Mao, Yuning  and
      Khabsa, Madian  and
      Lewis, Mike  and
      Hou, Rui  and
      Ba, Jimmy  and
      Almahairi, Amjad",
    booktitle = "Findings of the Association for Computational Linguistics: ACL",
    year = "2023",
    url = "https://aclanthology.org/2023.findings-acl.421",
    pages = "6740--6757",
}

@article{guo2024loptlowrankprompttuning,
  title={Lo{PT}: Low-Rank Prompt Tuning for Parameter Efficient Language Models},
  author={Guo, Shouchang and Damani, Sonam and Chang, Keng-hao},
  journal={arXiv preprint arXiv:2406.19486},
  year={2024},
  url={https://arxiv.org/abs/2406.19486}
}

@InProceedings{pmlr-v97-houlsby19a,
  title = 	 {Parameter-Efficient Transfer Learning for {NLP}},
  author =       {Houlsby, Neil and Giurgiu, Andrei and Jastrzebski, Stanislaw and Morrone, Bruna and De Laroussilhe, Quentin and Gesmundo, Andrea and Attariyan, Mona and Gelly, Sylvain},
  booktitle = 	 {Proceedings of the 36th International Conference on Machine Learning},
  pages = 	 {2790--2799},
  year = 	 {2019},
  url = 	 {https://proceedings.mlr.press/v97/houlsby19a.html},
}

@inproceedings{vu-etal-2022-spot,
    title = "{SP}o{T}: Better Frozen Model Adaptation through Soft Prompt Transfer",
    author = "Vu, Tu  and
      Lester, Brian  and
      Constant, Noah  and
      Al-Rfou{'}, Rami  and
      Cer, Daniel",
    booktitle = "Proceedings of the 60th Annual Meeting of the Association for Computational Linguistics",
    year = "2022",
    url = "https://aclanthology.org/2022.acl-long.346",
    pages = "5039--5059",
}

@inproceedings{NIPS2017_3f5ee243,
 author = {Vaswani, Ashish and Shazeer, Noam and Parmar, Niki and Uszkoreit, Jakob and Jones, Llion and Gomez, Aidan N and Kaiser, Lukasz and Polosukhin, Illia},
 booktitle = {Advances in Neural Information Processing Systems},
 title = {Attention is All you Need},
 url = {https://proceedings.neurips.cc/paper_files/paper/2017/file/3f5ee243547dee91fbd053c1c4a845aa-Paper.pdf},
 year = {2017}
}

@inproceedings{wang-etal-2018-glue,
    title = "{GLUE}: A Multi-Task Benchmark and Analysis Platform for Natural Language Understanding",
    author = "Wang, Alex  and
      Singh, Amanpreet  and
      Michael, Julian  and
      Hill, Felix  and
      Levy, Omer  and
      Bowman, Samuel",
    booktitle = "Proceedings of the 2018 {EMNLP} Workshop {B}lackbox{NLP}: Analyzing and Interpreting Neural Networks for {NLP}",
    year = "2018",
    url = "https://aclanthology.org/W18-5446",
    pages = "353--355",
}

@inproceedings{shin-etal-2020-autoprompt,
    title = "{A}uto{P}rompt: {E}liciting knowledge from language models with automatically generated prompts",
    author = "Shin, Taylor  and
      Razeghi, Yasaman  and
      Logan IV, Robert L.  and
      Wallace, Eric  and
      Singh, Sameer",
    booktitle = "Proceedings of the 2020 Conference on Empirical Methods in Natural Language Processing",
    year = "2020",
    url = "https://aclanthology.org/2020.emnlp-main.346",
    pages = "4222--4235",
}

@inproceedings{zhao2024tuning,
title={Tuning LayerNorm in Attention: Towards Efficient Multi-Modal {LLM} Finetuning},
author={Bingchen Zhao and Haoqin Tu and Chen Wei and Jieru Mei and Cihang Xie},
booktitle={International Conference on Learning Representations},
year={2024},
url={https://openreview.net/forum?id=YR3ETaElNK}
}

@inproceedings{ben-zaken-etal-2022-bitfit,
    title = "{B}it{F}it: Simple Parameter-efficient Fine-tuning for Transformer-based Masked Language-models",
    author = "Ben Zaken, Elad  and
      Goldberg, Yoav  and
      Ravfogel, Shauli",
    booktitle = "Proceedings of the 60th Annual Meeting of the Association for Computational Linguistics",
    year = "2022",
    url = "https://aclanthology.org/2022.acl-short.1",
    pages = "1--9",
}

@inproceedings{williams-etal-2018-broad,
    title = "A Broad-Coverage Challenge Corpus for Sentence Understanding through Inference",
    author = "Williams, Adina  and
      Nangia, Nikita  and
      Bowman, Samuel",
    booktitle = "Proceedings of the 2018 Conference of the North {A}merican Chapter of the Association for Computational Linguistics: Human Language Technologies",
    year = "2018",
    url = "https://aclanthology.org/N18-1101/",
    pages = "1112--1122",
}

@article{demszky2018transforming,
  title={Transforming question answering datasets into natural language inference datasets},
  author={Demszky, Dorottya and Guu, Kelvin and Liang, Percy},
  journal={arXiv preprint arXiv:1809.02922},
  year={2018},
  url={https://arxiv.org/abs/1809.02922}, 
}

@inproceedings{socher-etal-2013-recursive,
    title = "Recursive Deep Models for Semantic Compositionality Over a Sentiment Treebank",
    author = "Socher, Richard  and
      Perelygin, Alex  and
      Wu, Jean  and
      Chuang, Jason  and
      Manning, Christopher D.  and
      Ng, Andrew  and
      Potts, Christopher",
    booktitle = "Proceedings of the 2013 Conference on Empirical Methods in Natural Language Processing",
    year = "2013",
    url = "https://aclanthology.org/D13-1170/",
    pages = "1631--1642"
}

@inproceedings{cer-etal-2017-semeval,
    title = "{S}em{E}val-2017 Task 1: Semantic Textual Similarity Multilingual and Crosslingual Focused Evaluation",
    author = "Cer, Daniel  and
      Diab, Mona  and
      Agirre, Eneko  and
      Lopez-Gazpio, I{\~n}igo  and
      Specia, Lucia",
    booktitle = "Proceedings of the 11th International Workshop on Semantic Evaluation",
    year = "2017",
    url = "https://aclanthology.org/S17-2001/",
    pages = "1--14",
}

@inproceedings{dolan-brockett-2005-automatically,
    title = "Automatically Constructing a Corpus of Sentential Paraphrases",
    author = "Dolan, William B.  and
      Brockett, Chris",
    booktitle = "Proceedings of the Third International Workshop on Paraphrasing",
    year = "2005",
    url = "https://aclanthology.org/I05-5002/"
}

@article{warstadt2019neural_cola,
    title = "Neural Network Acceptability Judgments",
    author = "Warstadt, Alex  and
      Singh, Amanpreet  and
      Bowman, Samuel R.",
    journal = "Transactions of the Association for Computational Linguistics",
    year = "2019",
    url = "https://aclanthology.org/Q19-1040",
    pages = "625--641",
}

@inproceedings{giampiccolo-etal-2007-third,
    title = "The Third {PASCAL} Recognizing Textual Entailment Challenge",
    author = "Giampiccolo, Danilo  and
      Magnini, Bernardo  and
      Dagan, Ido  and
      Dolan, Bill",
    booktitle = "Proceedings of the {ACL}-{PASCAL} Workshop on Textual Entailment and Paraphrasing",
    year = "2007",
    url = "https://aclanthology.org/W07-1401",
    pages = "1--9",
}

@inproceedings{wang2019superglue,
 author = {Wang, Alex and Pruksachatkun, Yada and Nangia, Nikita and Singh, Amanpreet and Michael, Julian and Hill, Felix and Levy, Omer and Bowman, Samuel},
 booktitle = {Advances in Neural Information Processing Systems},
 title = {Super{G}LUE: A Stickier Benchmark for General-Purpose Language Understanding Systems},
 url = {https://proceedings.neurips.cc/paper_files/paper/2019/file/4496bf24afe7fab6f046bf4923da8de6-Paper.pdf},
 year = {2019}
}

@inproceedings{khashabi-etal-2018-looking,
    title = "Looking Beyond the Surface: A Challenge Set for Reading Comprehension over Multiple Sentences",
    author = "Khashabi, Daniel  and
      Chaturvedi, Snigdha  and
      Roth, Michael  and
      Upadhyay, Shyam  and
      Roth, Dan",
    booktitle = "Proceedings of the 2018 Conference of the North {A}merican Chapter of the Association for Computational Linguistics: Human Language Technologies",
    year = "2018",
    url = "https://aclanthology.org/N18-1023",
    pages = "252--262",
}

@inproceedings{clark-etal-2019-boolq,
    title = "{B}ool{Q}: Exploring the Surprising Difficulty of Natural Yes/No Questions",
    author = "Clark, Christopher  and
      Lee, Kenton  and
      Chang, Ming-Wei  and
      Kwiatkowski, Tom  and
      Collins, Michael  and
      Toutanova, Kristina",
    booktitle = "Proceedings of the 2019 Conference of the North {A}merican Chapter of the Association for Computational Linguistics: Human Language Technologies",
    year = "2019",
    url = "https://aclanthology.org/N19-1300",
    pages = "2924--2936",
}

@inproceedings{pilehvar-camacho-collados-2019-wic,
    title = "{W}i{C}: The Word-in-Context Dataset for Evaluating Context-Sensitive Meaning Representations",
    author = "Pilehvar, Mohammad Taher  and
      Camacho-Collados, Jose",
    booktitle = "Proceedings of the 2019 Conference of the North {A}merican Chapter of the Association for Computational Linguistics: Human Language Technologies",
    year = "2019",
    url = "https://aclanthology.org/N19-1128",
    pages = "1267--1273",
}

@inproceedings{levesque2012winograd_wnli,
  title={The {W}inograd schema challenge},
  author={Levesque, Hector and Davis, Ernest and Morgenstern, Leora},
  booktitle={Preceddings of the 13th International Conference on the Principles of Knowledge Representation and Reasoning},
  year={2012},
  url={https://cdn.aaai.org/ocs/4492/4492-21843-1-PB.pdf}
}

@inproceedings{de2019commitmentbank,
  title={The {C}ommitment{B}ank: Investigating projection in naturally occurring discourse},
  author={De Marneffe, Marie-Catherine and Simons, Mandy and Tonhauser, Judith},
  booktitle={Proceedings of Sinn und Bedeutung},
  pages={107--124},
  year={2019},
  url={https://semanticsarchive.net/Archive/Tg3ZGI2M/Marneffe.pdf}
}

@inproceedings{fisch-etal-2019-mrqa,
    title = "{MRQA} 2019 Shared Task: Evaluating Generalization in Reading Comprehension",
    author = "Fisch, Adam  and
      Talmor, Alon  and
      Jia, Robin  and
      Seo, Minjoon  and
      Choi, Eunsol  and
      Chen, Danqi",
    booktitle = "Proceedings of the 2nd Workshop on Machine Reading for Question Answering",
    year = "2019",
    url = "https://aclanthology.org/D19-5801",
    pages = "1--13",
}

@article{kwiatkowski-etal-2019-natural,
    title = "Natural Questions: A Benchmark for Question Answering Research",
    author = "Kwiatkowski, Tom  and
      Palomaki, Jennimaria  and
      Redfield, Olivia  and
      Collins, Michael  and
      Parikh, Ankur  and
      Alberti, Chris  and
      Epstein, Danielle  and
      Polosukhin, Illia  and
      Devlin, Jacob  and
      Lee, Kenton  and
      Toutanova, Kristina  and
      Jones, Llion  and
      Kelcey, Matthew  and
      Chang, Ming-Wei  and
      Dai, Andrew M.  and
      Uszkoreit, Jakob  and
      Le, Quoc  and
      Petrov, Slav",
    journal = "Transactions of the Association for Computational Linguistics",
    year = "2019",
    url = "https://aclanthology.org/Q19-1026/",
    pages = "452--466",
}

@inproceedings{yang-etal-2018-hotpotqa,
    title = "{H}otpot{QA}: A Dataset for Diverse, Explainable Multi-hop Question Answering",
    author = "Yang, Zhilin  and
      Qi, Peng  and
      Zhang, Saizheng  and
      Bengio, Yoshua  and
      Cohen, William  and
      Salakhutdinov, Ruslan  and
      Manning, Christopher D.",
    booktitle = "Proceedings of the 2018 Conference on Empirical Methods in Natural Language Processing",
    year = "2018",
    url = "https://aclanthology.org/D18-1259/",
    pages = "2369--2380",
}

@article{dunn2017searchqa,
  title={Search{QA}: A new {Q}\&{A} dataset augmented with context from a search engine},
  author={Dunn, Matthew and Sagun, Levent and Higgins, Mike and Guney, V Ugur and Cirik, Volkan and Cho, Kyunghyun},
  journal={arXiv preprint arXiv:1704.05179},
  year={2017},
  url={https://arxiv.org/abs/1704.05179}
}

@inproceedings{trischler-etal-2017-newsqa,
    title = "{N}ews{QA}: A Machine Comprehension Dataset",
    author = "Trischler, Adam  and
      Wang, Tong  and
      Yuan, Xingdi  and
      Harris, Justin  and
      Sordoni, Alessandro  and
      Bachman, Philip  and
      Suleman, Kaheer",
    booktitle = "Proceedings of the 2nd Workshop on Representation Learning for {NLP}",
    year = "2017",
    url = "https://aclanthology.org/W17-2623",
    pages = "191--200",
}

@article{10.1145/3474381,
author = {Sakaguchi, Keisuke and Bras, Ronan Le and Bhagavatula, Chandra and Choi, Yejin},
title = {Wino{G}rande: An Adversarial Winograd Schema Challenge at Scale},
year = {2021},
url = {https://doi.org/10.1145/3474381},
journal = {Communications of the ACM},
pages = {99–106},
}

@inproceedings{10.5555/2969239.2969312,
author = {Zhang, Xiang and Zhao, Junbo and LeCun, Yann},
title = {Character-Level Convolutional Networks for Text Classification},
year = {2015},
booktitle = {Advances in Neural Information Processing Systems},
pages = {649–657},
url={https://proceedings.neurips.cc/paper/2015/file/250cf8b51c773f3f8dc8b4be867a9a02-Paper.pdf},
}

@article{Khot_Sabharwal_Clark_2018, 
    title={SciTaiL: A Textual Entailment Dataset from Science Question Answering}, 
    url={https://ojs.aaai.org/index.php/AAAI/article/view/12022},
    journal={Proceedings of the AAAI Conference on Artificial Intelligence}, 
    author={Khot, Tushar and Sabharwal, Ashish and Clark, Peter}, 
    year={2018},
    volume={32},
    number={1},
}

@inproceedings{zhang-etal-2019-paws,
    title = "{PAWS}: Paraphrase Adversaries from Word Scrambling",
    author = "Zhang, Yuan  and
      Baldridge, Jason  and
      He, Luheng",
    booktitle = "Proceedings of the 2019 Conference of the North {A}merican Chapter of the Association for Computational Linguistics: Human Language Technologies",
    year = "2019",
    url = "https://aclanthology.org/N19-1131",
    pages = "1298--1308",
}

@inproceedings{ruckle-etal-2021-adapterdrop,
    title = "{AdapterDrop}: {O}n the Efficiency of Adapters in Transformers",
    author = {R{\"u}ckl{\'e}, Andreas  and
      Geigle, Gregor  and
      Glockner, Max  and
      Beck, Tilman  and
      Pfeiffer, Jonas  and
      Reimers, Nils  and
      Gurevych, Iryna},
    booktitle = "Proceedings of the 2021 Conference on Empirical Methods in Natural Language Processing",
    year = "2021",
    url = "https://aclanthology.org/2021.emnlp-main.626",
    pages = "7930--7946",
}

@inproceedings{mahabadi2021parameter,
    title = "Parameter-efficient Multi-task Fine-tuning for Transformers via Shared Hypernetworks",
    author = "Karimi Mahabadi, Rabeeh  and
      Ruder, Sebastian  and
      Dehghani, Mostafa  and
      Henderson, James",
    booktitle = "Proceedings of the 59th Annual Meeting of the Association for Computational Linguistics and the 11th International Joint Conference on Natural Language Processing",
    year = "2021",
    url = "https://aclanthology.org/2021.acl-long.47",
    pages = "565--576",
}

@inproceedings{ivison-peters-2022-hyperdecoders,
    title = "Hyperdecoders: Instance-specific decoders for multi-task {NLP}",
    author = "Ivison, Hamish  and
      Peters, Matthew",
    booktitle = "Findings of the Association for Computational Linguistics: EMNLP",
    year = "2022",
    url = "https://aclanthology.org/2022.findings-emnlp.124",
    pages = "1715--1730",
}

@inproceedings{sunglst,
title={{LST}: Ladder Side-Tuning for Parameter and Memory Efficient Transfer Learning},
author={Yi-Lin Sung and Jaemin Cho and Mohit Bansal},
booktitle={Advances in Neural Information Processing Systems},
year={2022},
url={https://openreview.net/forum?id=isPnnaTZaP5}
}

@inproceedings{muennighoff-etal-2023-crosslingual,
    title = "Crosslingual Generalization through Multitask Finetuning",
    author = "Muennighoff, Niklas  and
      Wang, Thomas  and
      Sutawika, Lintang  and
      Roberts, Adam  and
      Biderman, Stella  and
      Le Scao, Teven  and
      Bari, M Saiful  and
      Shen, Sheng  and
      Yong, Zheng Xin  and
      Schoelkopf, Hailey  and
      Tang, Xiangru  and
      Radev, Dragomir  and
      Aji, Alham Fikri  and
      Almubarak, Khalid  and
      Albanie, Samuel  and
      Alyafeai, Zaid  and
      Webson, Albert  and
      Raff, Edward  and
      Raffel, Colin",
    booktitle = "Proceedings of the 61st Annual Meeting of the Association for Computational Linguistics",
    year = "2023",
    url = "https://aclanthology.org/2023.acl-long.891/",
    pages = "15991--16111",
}

@inproceedings{wu2024reft,
title={Re{FT}: Representation Finetuning for Language Models},
author={Zhengxuan Wu and Aryaman Arora and Zheng Wang and Atticus Geiger and Dan Jurafsky and Christopher D Manning and Christopher Potts},
booktitle={Advances in Neural Information Processing Systems},
year={2024},
url={https://openreview.net/forum?id=fykjplMc0V}
}

@inproceedings{yin2024lofit,
title={Lo{F}i{T}: Localized Fine-tuning on {LLM} Representations},
author={Fangcong Yin and Xi Ye and Greg Durrett},
booktitle={Advances in Neural Information Processing Systems},
year={2024},
url={https://openreview.net/forum?id=dfiXFbECSZ}
}

@inproceedings{10.1145/502512.502546,
author = {Bingham, Ella and Mannila, Heikki},
title = {Random projection in dimensionality reduction: applications to image and text data},
year = {2001},
url = {https://doi.org/10.1145/502512.502546},
booktitle = {Proceedings of the Seventh ACM SIGKDD International Conference on Knowledge Discovery and Data Mining},
pages = {245–250},
}

@inproceedings{hao2024flora,
title={Flora: Low-Rank Adapters Are Secretly Gradient Compressors},
author={Yongchang Hao and Yanshuai Cao and Lili Mou},
booktitle={Proceedings of the 41st International Conference on Machine Learning},
year={2024},
url={https://proceedings.mlr.press/v235/hao24a.html},
pages = {17554--17571}
}

@inproceedings{kopiczko2024vera,
title={Ve{RA}: Vector-based Random Matrix Adaptation},
author={Dawid Jan Kopiczko and Tijmen Blankevoort and Yuki M Asano},
booktitle={International Conference on Learning Representations},
year={2024},
url={https://openreview.net/forum?id=NjNfLdxr3A}
}

@inproceedings{
wu2024mixture,
title={Mixture of {L}o{RA} Experts},
author={Xun Wu and Shaohan Huang and Furu Wei},
booktitle={International Conference on Learning Representations},
year={2024},
url={https://openreview.net/forum?id=uWvKBCYh4S}
}

@inproceedings{
lialin2024relora,
title={Re{L}o{RA}: High-Rank Training Through Low-Rank Updates},
author={Vladislav Lialin and Sherin Muckatira and Namrata Shivagunde and Anna Rumshisky},
booktitle={International Conference on Learning Representations},
year={2024},
url={https://openreview.net/forum?id=DLJznSp6X3}
}

@inproceedings{hayou2024lora,
title={Lo{RA}+: Efficient Low Rank Adaptation of Large Models},
author={Soufiane Hayou and Nikhil Ghosh and Bin Yu},
booktitle={Proceedings of the 41st International Conference on Machine Learning},
year={2024},
url={https://proceedings.mlr.press/v235/hayou24a.html},
pages = {17783--17806}
}

@article{qin2022exploringuniversalintrinsictask,
  title={Exploring universal intrinsic task subspace via prompt tuning},
  author={Yujia Qin and Xiaozhi Wang and Yusheng Su and Yankai Lin and Ning Ding and Jing Yi and Weize Chen and Zhiyuan Liu and Juanzi Li and Lei Hou and Peng Li and Maosong Sun and Jie Zhou},
  journal={arXiv preprint arXiv:2110.07867},
  year={2021},
  url={https://arxiv.org/abs/2110.07867}
}

@article{cobbe2021gsm8k,
  title={Training Verifiers to Solve Math Word Problems},
  author={Cobbe, Karl and Kosaraju, Vineet and Bavarian, Mohammad and Chen, Mark and Jun, Heewoo and Kaiser, Lukasz and Plappert, Matthias and Tworek, Jerry and Hilton, Jacob and Nakano, Reiichiro and Hesse, Christopher and Schulman, John},
  journal={arXiv preprint arXiv:2110.14168},
  year={2021},
  url={https://arxiv.org/abs/2110.14168}
}

@article{austin2021program,
  title={Program Synthesis with Large Language Models},
  author={Austin, Jacob and Odena, Augustus and Nye, Maxwell and Bosma, Maarten and Michalewski, Henryk and Dohan, David and Jiang, Ellen and Cai, Carrie and Terry, Michael and Le, Quoc and Sutton, Charles},
  journal={arXiv preprint arXiv:2108.07732},
  year={2021},
  url={https://arxiv.org/pdf/2108.07732}
}

@article{grattafiori2024llama,
  title={The llama 3 herd of models},
  author={AI @ Meta},
  journal={arXiv preprint arXiv:2407.21783},
  year={2024},
  url={https://arxiv.org/abs/2407.21783}
}

@inproceedings{gao2024parameterefficient,
title={Parameter-Efficient Fine-Tuning with Discrete {F}ourier Transform},
author={Ziqi Gao and Qichao Wang and Aochuan Chen and Zijing Liu and Bingzhe Wu and Liang Chen and Jia Li},
booktitle={Proceedings of the 41st International Conference on Machine Learning},
year={2024},
url={https://openreview.net/forum?id=XUOHKSsurt},
pages = {14884--14901}
}

@inproceedings{pan2024lisa,
title={{LISA}: Layerwise Importance Sampling for Memory-Efficient Large Language Model Fine-Tuning},
author={Rui Pan and Xiang Liu and Shizhe Diao and Renjie Pi and Jipeng Zhang and Chi Han and Tong Zhang},
booktitle={Advances in Neural Information Processing Systems},
year={2024},
url={https://openreview.net/forum?id=L8ifDX5XNq},
pages = {57018--57049},
}

@inproceedings{liu-etal-2024-customizing-large,
    title = "Customizing Large Language Model Generation Style using Parameter-Efficient Finetuning",
    author = "Liu, Xinyue  and
      Diddee, Harshita  and
      Ippolito, Daphne",
    booktitle = "Proceedings of the 17th International Natural Language Generation Conference",
    year = "2024",
    url = "https://aclanthology.org/2024.inlg-main.34/",
    pages = "412--426",
}

@inproceedings{lan-etal-2025-efficient,
    title = "Efficient and Effective Prompt Tuning via Prompt Decomposition and Compressed Outer Product",
    author = "Lan, Pengxiang  and
      Xu, Haoyu  and
      Yang, Enneng  and
      Liang, Yuliang  and
      Guo, Guibing  and
      Zhao, Jianzhe  and
      Wang, Xingwei",
    booktitle = "Proceedings of the 2025 Conference of the Nations of the Americas Chapter of the Association for Computational Linguistics",
    year = "2025",
    url = "https://aclanthology.org/2025.naacl-long.225/",
    pages = "4406--4421",
}

@inproceedings{liu2022fewshot,
title={Few-Shot Parameter-Efficient Fine-Tuning is Better and Cheaper than In-Context Learning},
author={Haokun Liu and Derek Tam and Muqeeth Mohammed and Jay Mohta and Tenghao Huang and Mohit Bansal and Colin Raffel},
booktitle={Advances in Neural Information Processing Systems},
year={2022},
url={https://openreview.net/forum?id=rBCvMG-JsPd}
}
